\documentclass[]{article}

\usepackage[T1]{fontenc}

\usepackage{babel}
\usepackage{bm}
\usepackage{graphicx} % Required for inserting images
\usepackage{amsmath}
\usepackage{amssymb}
\usepackage{amsthm}
\usepackage{xcolor}
\usepackage{hyperref}
\usepackage[margin=1.5in]{geometry}
\usepackage{url}
\usepackage{eucal}
\usepackage{lipsum}
\usepackage{amsfonts}
\usepackage{graphicx}
\usepackage{epstopdf}
\usepackage{algorithmic}
\usepackage{enumitem}
\usepackage{amsopn}
\usepackage[numbers]{natbib}

\setlist[enumerate]{leftmargin=.5in}
\setlist[itemize]{leftmargin=.5in}

\title{Subspace Langevin Monte Carlo}

\author{Tyler  Maunu \qquad Jiayi Yao \\Brandeis University, Waltham,  MA }

\newcommand{\R}{\mathbb R}

\newcommand{\N}{\mathbb N}
\newcommand{\E}{\mathbb E}

\def\bA{\boldsymbol{A}}
\def\bB{\boldsymbol{B}}

\def\bD{\boldsymbol{D}}

\def\bG{\boldsymbol{G}}

\def\bI{\boldsymbol{I}}

\def\bP{\boldsymbol{P}}

\def\bU{\boldsymbol{U}}

\def\bW{\boldsymbol{W}}

\def\cE{\mathcal{E}}
\def\cF{\mathcal{F}}
\def\cG{\mathcal{G}}
\def\cH{\mathcal{H}}

\def\cL{\mathcal{L}}

\def\cP{\mathcal{P}}

\def\bzero{\boldsymbol{0}}
\def\bone{\boldsymbol{1}}

\def\bSigma{\boldsymbol{\Sigma}}

\def\diag{\mathrm{diag}}

\DeclareMathOperator{\Tr}{\text{Tr}}

\DeclareMathOperator{\Exp}{\mathsf{Exp}}
\DeclareMathOperator{\Log}{\mathsf{Log}}

\DeclareMathOperator{\id}{\mathsf{id}}
\DeclareMathOperator{\di}{\mathrm{d} \!}

\DeclareMathOperator{\diver}{\text{div}}
\DeclareMathOperator{\KL}{\mathsf{KL}}
\DeclareMathOperator{\Law}{\mathsf{Law}}

\DeclareMathOperator{\rel}{\mathsf{rel}}

\newtheorem{lem}{Lemma}
\newtheorem{defn}{Definition}
\newtheorem{assump}{Assumption}

\newtheorem{thm}{Theorem}

\newtheorem{prop}{Proposition}
\newtheorem{rem}{Remark}

\begin{document}

\maketitle

\begin{abstract}  
    Sampling from high-dimensional distributions has wide applications in data science and machine learning but poses significant computational challenges. We introduce Subspace Langevin Monte Carlo (SLMC), a novel and efficient sampling method that generalizes random-coordinate Langevin Monte Carlo and preconditioned Langevin Monte Carlo by projecting the Langevin update onto subsampled eigenblocks of a time-varying preconditioner at each iteration. The advantage of SLMC is its superior adaptability and computational efficiency compared to traditional Langevin Monte Carlo and preconditioned Langevin Monte Carlo. Using coupling arguments, we establish error guarantees for SLMC and demonstrate its practical effectiveness through a few experiments on sampling from ill-conditioned distributions.
\end{abstract}

\section{Introduction}

The Langevin diffusion and its variants have become a fundamental object of study in modern machine learning. Practically, discretizations of the Langevin diffusion are highly scalable for generating samples from complex, high-dimensional target distributions. Successful applications of these methods include denoising diffusion models \cite{ho2020denoising,song2021score}, characterization of complex Bayesian posteriors \cite{durmus2019high}, and differential privacy mechanisms \cite{chourasia2021differential}. On the mathematical side, these diffusions are deeply connected to Wasserstein gradient flows. This connection has been used to study their convergence and, consequently, to develop new and more efficient diffusions. 

As hinted at in the previous paragraph, the Langevin diffusion solves an optimization problem in Wasserstein space. Taking this point of view, a theme in recent research has been to translate ideas from gradient-based optimization over Euclidean space to optimization over Wasserstein space. In particular, this has resulted in a growing zoo of optimization methods over Wasserstein space, including notions of gradient descent \cite{santambrogio2017euclidean}, mirror descent and Newton methods \cite{hsieh2018mirrored,zhang2020wasserstein,chewi2020exponential,wang2020information}, proximal algorithms \cite{chen2022improved}, preconditioned methods \cite{titsias2024optimal,bonet2024mirror}, and more.

One class of methods that is recently gaining prominence in the optimization literature is block coordinate and subspace descent methods \cite{nesterov2012efficiency,gower2019rsn,hanzely2020stochastic,kozak2021stochastic,cosson2023low,feinberg2024sketchy,zhao2024galore,liang2024memory}. These methods incorporate low-dimensional updates and, therefore, can be highly efficient and scalable in high dimensions. Projected optimization methods have advantages in terms of the number of function evaluations \cite{kozak2021stochastic}, directional derivative oracle calls \cite{cosson2023gradient}, and reduced memory for storing adaptive preconditioners \cite{feinberg2024sketchy,zhao2024galore,liang2024memory}. Furthermore, higher-order information can be efficiently incorporated; see, for example, \cite{hanzely2020stochastic}. Analogs of most of these methods have not yet been extended to the Wasserstein space outside the random coordinate version of Langevin Monte Carlo (RCLMC) \cite{ding2021random,ding2021langevin}.

In many modern applications, the Langevin algorithm is run in high-dimensional spaces. Therefore, it stands to reason that it could benefit from further studies into more efficient updating schemes and extending the ideas of RCLMC \cite{ding2021langevin}. Our work represents a significant step towards achieving this goal and understanding how to implement subspace methods in Wasserstein space. Our key contributions are as follows: 
\begin{enumerate}
    \item We present a novel formulation for projected Wasserstein gradient flows. By projecting the Wasserstein gradient to a lower-dimensional subspace, it is possible to reduce the complexity of flows while maintaining convergence guarantees in continuous time. After giving an abstract theoretical guarantee, we focus on a specific case of minimizing the KL divergence.
    \item We develop a random subspace discretization of the Wasserstein gradient flow of the Kullback-Leibler (KL) divergence called Subspace Langevin Monte Carlo (SLMC). SLMC significantly generalizes past work on Random Coordinate Langevin Monte Carlo (RCLMC) \cite{ding2021random} since it uses a low-rank preconditioner rather than just updating individual coordinates. A particular subcase of SLMC is block-coordinate LMC, which, to our knowledge, has not received attention in the literature.
    \item We conduct a coupling argument that gives error guarantees for SLMC under certain assumptions. We show that the bounds for SLMC outperform past bounds for other variants of LMC in specific regimes. Furthermore, SLMC can be implemented with reduced memory, which allows it to better adapt to high-dimensional settings.
    \item We finish with simple experiments on sampling from ill-conditioned distributions. Our examples include sampling from a Gaussian distribution, a funnel distribution, and a posterior in Bayesian logistic regression. These experiments demonstrate the flexibility and adaptability of SLMC for fast sampling in such settings.
\end{enumerate}

\subsection{Structure of Paper}

We now outline the structure of our paper. In Section \ref{sec:background}, we cover the necessary background to understand our work and give an overview of related work to place it in a broader context. Then, in Section \ref{sec:methods}, we discuss the SLMC algorithm as a generalization of past LMC methods. Section \ref{sec:theory} contains our theoretical analysis of the SLMC method, a discussion of the mathematical techniques, and a comparison to other methods. Finally, Section \ref{sec:exp} presents experiments on an ill-conditioned Gaussian distribution, a posterior in Bayesian logistic regression, and a funnel distribution.

\subsection{Notation}

We seek to generate samples from a distribution over $\R^d$ with density $\pi(\cdot) \propto \exp(-V(\cdot))$. For a positive definite matrix $\bA\succ \bzero$, we define an inner product $\langle x, y \rangle_{\bA} = x^T \bA y$ and its associated norm $\|x\|_{\bA}=\sqrt{\langle x, x\rangle_{\bA}}$. Random vectors are denoted with upper case letters, and matrices are denoted with bold upper case letters. Fixed vectors and scalars are lowercase letters. The Euclidean norm for vectors is $\|\cdot\|$, and the spectral (operator) norm for a matrix $\bA$ is $\|\bA\|_2$. The $\ell \times \ell$ identity matrix is $\bI_{\ell}$, and we will suppress the subscript when the dimension is obvious from context. For symmetric matrices, $\succeq$ will denote the Loewner ordering: $\bA \succeq \bB$ if $\bA - \bB$ is positive semidefinite. For a random vector $X$, $\Law(X)$ denotes its associated probability measure. The set of semiorthogonal matrices is $O(d,r) = \{\bU \in \R^{d \times r} : \bU^T \bU = \bI_r\}$.

\section{Background}
\label{sec:background}

We begin our paper by giving the necessary mathematical background to understand our work. Section \ref{subsec:wgf} discusses Wasserstein space and Wasserstein gradient flows. Then, Section \ref{subsec:lmc} discusses how the Langevin diffusion can simulate a Wasserstein gradient flow. Section \ref{subsec:plmc} shows how to generalize these gradient flows to incorporate preconditioning. Then, in Section \ref{subsec:euclcase}, we discuss the idea of subspace descent in Euclidean space. We finish this section with an overview of related work in Section \ref{subsec:related}.

\subsection{Wasserstein Geometry and Gradient Flows}
\label{subsec:wgf}

We now briefly outline Wasserstein gradient flows and their connection to Langevin Monte Carlo. We refer the reader to \cite{ambrosio2008gradient,villani2009optimal} for a more detailed discussion of these topics. For a friendly introduction to Wasserstein gradient flows, we recommend \cite{santambrogio2017euclidean,chewi2024log}. Throughout, we restrict to measures with finite second moment and densities with respect to the Lebesgue measure, $\cP_{2,ac}(\R^d)$. The 2-Wasserstein distance between measures $\mu$ and $\nu$ is
\begin{equation}\label{eq:wassdist}
    W_2^2(\mu, \nu) = \inf_{\gamma \in \Gamma(\mu, \nu)} \int \|x-y\|^2 d\gamma(x,y),
\end{equation}
where $\Gamma(\mu, \nu)$ is the set of all couplings between $\mu$ and $\nu$. The infimum is achieved at an optimal coupling $\gamma_\star$, and if $\mu$ is absolutely continuous $\gamma_\star = (\id, T_{\mu \to \nu})_{\#} \mu$, where $T_{\mu \to \nu}$ is an optimal transport map from $\mu$ to $\nu$. The set $\cP_{2,ac}$ endowed with this metric defines a geodesic metric space known as 2-Wasserstein space, or for the present discussion just Wasserstein space. Wasserstein space has a formal differential geometric structure, where the tangent space at an absolutely continuous measure $\mu \in \cP_{2,ac}(\R^d)$ is identified with the set of displacements, 
\[
    T_{\mu} \cP_{2}(\R^d) = \overline{\{\lambda(T-\id) : T \text{ OT map}, \ \lambda \geq 0\}}^{L^2(\mu)}.
\]
The logarithmic and exponential maps are
\[
    \Exp_{\mu}(v) = [\id + v]_{\#} \mu, \ \Log_{\mu}\nu = T_{\mu \to \nu} - \id,
\]
where $T_{\mu \to \nu}$ is the optimal transport map from $\mu$ to $\nu$. 

The function we consider minimizing over Wasserstein space is the KL divergence between a variable measure $\mu$ and the target measure $\pi \propto \exp(-V)$ that we wish to sample from. This objective is defined as
\begin{equation}\label{eq:kl}
    \cF(\mu) = \KL(\mu | \pi) = \int \log \left(\frac{\mu}{\pi}\right)  \mu.
\end{equation}
This setup is typical in sampling problems since we can optimize $\cF$ without knowing the normalizing constant of $\pi$.
Using the formal differential structure mentioned above, one can define a Wasserstein gradient flow associated with the minimization of this functional \cite[Section 4.3]{santambrogio2017euclidean} through the continuity equation
\begin{equation}\label{eq:wgf}
    \partial_t \mu_t =  \diver\Big(\mu_t \nabla \log \frac{\mu_t}{\pi}\Big).
\end{equation}
Here, $\nabla \log \frac{\mu_t}{\pi} = \nabla_{W_2} \cF(\mu_t)$ is the Wasserstein gradient of $\cF$, which is found via the Euclidean gradient of the first variation of $\cF$.

\subsection{Langevin Monte Carlo}
\label{subsec:lmc}

While the previous interpretation describes the evolution of densities using a PDE, we can also take a Lagrangian point of view and look at the evolution of particles $X_t \sim \mu_t$. It turns out that the Langevin diffusion encodes this evolution through the SDE \cite{jordan1998variational}
\begin{equation}\label{eq:ld}
    d X_t = - \nabla V(X_t) dt + \sqrt{2} dB_t,
\end{equation}
where $B_t$ is standard Brownian motion.
In particular, one can show that $\Law(X_t) = \mu_t$, where $\mu_t$ solves \eqref{eq:wgf}. Therefore, one can generate samples from $\pi$ by solving this stochastic differential equation.
The Euler-Maruyama discretization of this SDE is known as Langevin Monte Carlo ({\sf LMC})
\begin{equation}\label{eq:lmc}
    X_{k+1} = X_k - h_k\nabla V(X_k) + \sqrt{2h_k} \xi_k,
\end{equation}
where $\xi_k \sim N(0, I)$. One can view this as a discretization of \eqref{eq:wgf} using a forward-flow \cite{wibisono2018sampling}. 

To understand the theoretical performance of LMC, we consider a simplified setting of log-concave sampling \cite{chewi2024log}. We say that the potential $V$ is $\alpha$-strongly convex and $\beta$-smooth if
\[
    V(x) + \langle \nabla V(x), y-x\rangle + \frac{\alpha}{2} \|y-x\|^2 \leq V(y) \leq V(x) + \langle \nabla V(x), y-x\rangle + \frac{\beta}{2} \|y-x\|^2.
  \]
State-of-the-art bounds for mixing LMC in this setting can be found in \cite{durmus2019analysis}. For reasons that become more apparent later, we measure the iteration complexity of a method in terms of the number of calls it makes to a directional derivative oracle. For LMC in this setting, to have $W_2(\mu_{N_{\mathrm{LMC}}}, \pi) \leq \epsilon$, one needs 
\begin{equation}\label{eq:itcomplmc}
    N_{\mathrm{LMC}} = O\left(\frac{d^2 \beta}{\epsilon^2 \alpha^2 } \log \frac{W_2(\mu_0, \pi)}{\epsilon} \right).
\end{equation}
Here, $\beta/\alpha$ is the condition number, and the extra factor of $d$ over the result in \cite{durmus2019analysis} comes from the computation of $d$ directional derivatives at each iteration of the method.

\subsection{Preconditioned Langevin Methods}
\label{subsec:plmc}

This section discusses preconditioned gradient flows of the KL divergence and the preconditioned Langevin algorithm. Our main application in this work is developing efficient subspace approximations to these.

We define a preconditioned version of the Wasserstein gradient flow of the KL divergence \eqref{eq:wgf} as 
\begin{equation}\label{eq:pcwgf}
    \partial_t \mu_t =  \diver\Big(\mu_t \bA_t \nabla \log \frac{\mu_t}{\pi}\Big),
\end{equation}
for a sequence of positive definite matrices $\bA_t \succ \bzero$. \cite{ma2015complete} studied a more general version of this method with spatially varying preconditioners, and more recently, it has been shown that preconditioned methods can be more efficient than standard LMC \cite{bhattacharya2023fast,titsias2024optimal}. However, to avoid dealing with an extra bias term \cite{ma2015complete}, we will assume that $\bA_t$ is a function of time alone but not space.

If $\bA_t$ is fixed, then one can view preconditioned Langevin diffusion as a mirror Langevin diffusion with a linear mirror map \cite{zhang2020wasserstein,chewi2020exponential,ahn2021efficient}. 
Alternatively, this is a Wasserstein gradient flow with respect to $W_{2, \bA^{-1}}$, which is $W_2$ as defined in \eqref{eq:wassdist} with the norm replaced by $\|\cdot\|_{\bA^{-1}}$. 

The preconditioned Langevin diffusion corresponding to the flow \eqref{eq:pcwgf} is
\begin{equation}\label{eq:pld}
    dX_t =- \bA_t\nabla V(X_t) dt + \sqrt{2} \bA_t^{1/2} dB_t.
\end{equation}
We refer to the discretization of this as Preconditioned Langevin Monte Carlo (PLMC), 
\begin{equation}\label{eq:plmc}
    X_{k+1} =X_k- h_k\bA_k \nabla V(X_k) + \sqrt{2h_k} \bA_k^{1/2} \xi_k.
\end{equation}

For fixed $\bA$, the diffusion \eqref{eq:pld} has $\pi$ as a stationary distribution because it is a special case of Theorem 1 of \cite{ma2015complete}.
\begin{thm}[Theorem 1 of \cite{ma2015complete}]
    If $\bA_t = \bA(X_t) \succeq \bzero$ is a function of the spatial coordinate alone, then \eqref{eq:pld} has $\pi \propto \exp(-V)$ as a stationary distribution.
\end{thm}
While it is feasible to show stationarity for $\bA_t$ that is piecewise constant in time, which is sufficient for discrete time analysis, we leave to future work the development of conditions on the time-varying preconditioners such that $\pi$ is stationary.

\begin{rem}
The stationarity analysis does not apply to $\bA_t$ that are allowed to depend on $X_t$, which is the case in adaptive or Riemannian Langevin methods \cite{yu2024scalable}. This is due to the need for an additional correction term. 
\end{rem}

\subsection{Euclidean Subspace Descent}
\label{subsec:euclcase}

The main idea of the present work is to use low-rank matrices $\bA_t$ rather than full-rank matrices.
To this end, we review subspace descent for Euclidean optimization. Instead of sampling from a density, suppose we wish to minimize a function $f$ over $\R^d$. 

Subspace descent methods in Euclidean space have both continuous and discrete time formulations. In continuous time, for a sequence of projection matrices $\bP_t$, $t \in [0, \infty)$, the subspace gradient flow of $f$ initialized at $x_0$ solves the ODE
\begin{equation}\label{eq:projflow}
    \dot x_t = - \bP_t \nabla f(x_t).
\end{equation}
The forward Euler discretization of this flow is
\begin{equation}\label{eq:discrprojflow}
    x_{k+1} = x_{k} - h \bP_k \nabla f(x_k),
\end{equation}
where $(\bP_k)_{k \in \N}$ is a sequence of low-rank matrices and $h$ is the step size. The convergence of the discrete method is considered in \cite{kozak2021stochastic}. They assume that $\bP_k$ are scaled orthogonal projections onto an $r$-dimensional subspace such that $\E \bP_k = \bI$. Furthermore, it is assumed that $f$ is $\beta$-smooth and satisfies a Polyak-\L{}ojasiewicz (PL) inequality, $f(x) - f(x_\star) \leq \frac{1}{2\alpha} \|\nabla f(x)\|^2$. Under these conditions, one can show  
\begin{equation}
    \E f(x_k) - f(x_\star) \leq \omega^k (f(x_0) - f(x_\star))
\end{equation}
for $\omega = 1-\frac{r \alpha}{ d \beta}$. This implies an iteration complexity of $k=O(\frac{d\beta}{r \alpha} \log(1/\epsilon))$ to achieve an expected optimality gap $\E f(x_k) - f(x_\star) \leq \epsilon$. 

There is no free lunch regarding computational cost here. The per-iteration complexity of subspace gradient descent requires $r$ directional derivative computations. In contrast, the per-iteration complexity of gradient descent requires $d$ directional derivatives. Therefore, both methods require $O(d \beta/\alpha \log(1/\epsilon))$ directional derivative computations to achieve an $\epsilon$-accurate solution. Although this analysis suggests comparable complexity, subspace methods offer several key practical advantages.

First, the smoothness parameter $\beta$ can be traded for a potentially smaller directional smoothness. This is analogous to what is done in coordinate descent methods \cite{wright2015coordinate}. The second advantage is in terms of memory. In particular, the method does not need to store the full $d$-dimensional gradient at each iteration; it can just store the $r$-dimensional directional gradient.  
Third, in settings such as PDE-constrained optimization \cite{kozak2021stochastic}, computing more directional derivatives typically involves higher computational cost due to needing to solve the PDE.

We finish with a brief discussion of continuous time convergence of subspace descent over Euclidean spaces. We can establish a continuous-time convergence result analogous to the discrete-time case for subspace descent
\begin{prop}\label{prop:ctssubdescent}
    Suppose that $f$ satisfies a PL inequality with parameter $\alpha$. Then, the continuous time flow $\eqref{eq:projflow}$ satisfies
    \[ 
        f(x_t) - f(x_\star) \leq \exp\left( -\alpha \int_0^t m_s  ds \right) (f(x_0) - f(x_\star)),
    \]
    for any continuous $m_t$ such that
    \[
         0 \leq m_t \leq \frac{\langle \nabla f(x_t), \bP_t \nabla f(x_t) \rangle}{\|\nabla f(x_t)\|^2} .
    \]
    Furthermore, if $\E \bP_t = c \bI$, then
    \[
        \E f(x_t) - f(x_\star) \leq \exp\left( - t \alpha c \right) (f(x_0) - f(x_\star)).
    \]
\end{prop}
This result establishes that we can achieve convergence rates in the continuous-time setting even with projected updates.

In expectation, there is no loss in the convergence rate when compared with the standard result for gradient flow under a PL inequality, provided that $\E \bP_t = \bI$. Furthermore, if the gradient of $f$ is close to the span of $\bP_t$ and $m_t$ is close to 1, then we see that the convergence rate for the subspace flow is close to that of the deterministic full-dimensional gradient flow.

\subsection{Review of Related Work}
\label{subsec:related}

The works most directly related to ours are the sequence of works that studied Random Coordinate Langevin Monte Carlo (RCLMC) and its reduced-variance versions \cite{ding2021random,ding2021langevin}. These methods randomly sample coordinates and take a Langevin step along that coordinate only at each iteration. Another work in this vein is \cite{roy2022stochastic}, which uses random directions to compute zeroth-order approximations to the gradient that are then used for a Langevin method.

Subspace Diffusion Models \cite{jing2022subspace} learn denoising diffusion models over a sequence of nested subspaces. Since the SDE in each time interval has a solution constrained to each subspace, one can learn the score function restricted to these subspaces, which results in greater efficiency. Other recent work explores emergent low dimensionality within diffusion models. For example, \cite{wang2024diffusion} shows that diffusion models trained on low-rank Gaussian mixtures do not incur the curse of dimensionality. \cite{chen2024exploring} uses low-rank structures of the Jacobian of the posterior mean prediction to edit diffusion models.

Random projections have a long history in data science, beginning with the seminal work of \cite{johnson1986extensions}. Variants of random projection and sketching algorithms include CountSketch \cite{charikar2004finding}, sparse JL transform \cite{dasgupta2010sparse}, Subsampled Randomized Hadamard Transform \cite{ailon2009fast,tropp2011improved}, and more. See \cite{halko2011finding} for an overview. 

Our work draws inspiration from subspace methods for optimization. The oldest work on random subspace methods are random coordinate and block coordinate descent; see the discussion in \cite{nesterov2012efficiency,wright2015coordinate}.  Early works \cite{gower2015randomized,gower2015stochastic} use randomized methods for solving linear systems. \cite{frongillo2015convergence} analyzes subspace descent as a unification of coordinate and block coordinate descent methods. \cite{nesterov2017random} uses random Gaussian search directions for gradient-free optimization.
\cite{kozak2021stochastic} generalizes this to random subspaces and studies the effect of random projections on the convergence gradient descent. \cite{gower2019rsn} studies a Newton method on random subspaces, and \cite{hanzely2020stochastic} develops a random subspace-based Cubic Newton Method. Other recent work uses sketching for more efficient SDP solvers \cite{yurtsever2021scalable}. \cite{ivkin2019communication} uses gradient sketches for efficient distributed optimization. Variance reduction techniques can be used with gradient sketching as well \cite{hanzely2018sega,gower2021stochastic}.

Other works consider the estimation of a suitable subspace rather than random subspaces for optimization. \cite{cosson2023gradient} computes PCA on sampled gradients of a function and then runs gradient descent constrained to the found subspace. Other works seek to find memory-efficient methods for training large language models by using adaptive projections
\cite{feinberg2024sketchy,zhao2024galore,liang2024memory}. %Our work can be seen as a hybrid of random subspace updates with some adaptation since the subspaces are taken from the eigendecomposition of a preconditioning matrix. 

Some work has studied the low-dimensionality of stochastic gradient descent iterates in high-dimensional nonconvex optimization. Earlier works of \cite{sagun2016eigenvalues,sagun2017empirical} show this was empirically the case. Later, \cite{jastrzkebski2019relation} demonstrates that SGD dynamics tend to follow sharp directions in the loss landscape. \cite{arous2024high} shows that training aligns with principal subspaces of the Hessian or Gram matrices in multi-class logistic regression or XOR classification. On the other hand,
\cite{song2024does} argues that for general deep learning tasks, learning by SGD does not occur in the top subspace of the Hessian but rather in the bulk subspace. \cite{li2018measuring} and \cite{gressmann2020improving} use projections to random subspaces to get around this, and \cite{li2022low} uses a carefully chosen subspace to improve DNN training.

Another strategy to make sampling more efficient involves running a diffusion in a latent space \cite{vahdat2021score,rombach2022high}. However, one must train a variational autoencoder to perform the embedding. This additional step has made it difficult to give proper theoretical bounds on the mixing time of the resulting diffusion method, although some results exist \cite{tzen2019theoretical}. Additionally, the dimension of latent spaces is typically still high.

Finally, random projections have been used in the study of Wasserstein space, particularly in the study of Sliced Wasserstein Distances \cite{bonneel2015sliced}. These offer a scalable way of estimating distances between high-dimensional distributions and have associated gradient flows \cite{bonet2022efficient}. We note that this path is distinct from the one we take here: \emph{\cite{bonet2022efficient} developed gradient flows of the space of probability measures equipped with a different metric (the sliced Wasserstein distance). We seek to approximate (preconditioned) Wasserstein gradient flows themselves with random projections.}

\section{The Subspace Langevin Algorithm}
\label{sec:methods}

As discussed in the introduction, variants of subspace descent have been used to scale optimization methods in high-dimensional settings. Our focus in this paper is to extend subspace descent to optimization over Wasserstein space by focusing first on the particular problem of minimizing the KL divergence between a proposal and a target measure. 

We begin by outlining how to define subspace gradient flows in Wasserstein space and give a simple proof of convergence for them in Section \ref{subsec:continuoustime}. After this, Section \ref{subsec:slmc} specializes this by defining SLMC as a discretization of a specific subspace flow in Wasserstein space. The resulting algorithm and its theoretical analysis are the main contributions of our work.

\subsection{Subspace Gradient Flows in Wasserstein Space}
\label{subsec:continuoustime}

In this section, we will outline subspace gradient flows in Wasserstein space. After defining them, we extend the continuous time analysis of Euclidean subspace descent presented in Section \ref{subsec:euclcase} to give a quantitative convergence guarantee. We finish by discussing a specific function that is our primary object of interest. 

Concretely, we consider an abstract optimization problem over the space of measures,
\[
    \min_{\mu} \cG(\mu). 
\]
We assume that $\cG$ is $\alpha$-strongly geodesically convex over Wasserstein space,
\[
    \cG(\nu)  \geq \cG(\mu) + \langle \nabla_{W_2} \cG(\mu), \Log_\mu \nu \rangle_\mu + \frac{\alpha}{2} W_2^2(\mu, \nu),
\]
where $\alpha > 0$ and $\|\cdot\|_{\mu}$ is the norm defined by the inner product on $L^2(\mu)$. This implies that $\cG$ has a unique minimizer $\mu_\star$.
This implies that $\cG$ satisfies a PL inequality over Wasserstein space,
\begin{equation}\label{eq:wasspl}  
    \cG(\mu) - \cG(\mu_\star) \leq \frac{1}{2\alpha} \|\nabla_{W_2} \cG(\mu)\|_{\mu}^2.
\end{equation}

In continuous time, we can formulate a projected flow as
\[
    \partial_t \mu_t - \nabla \cdot (\mu_t \bP_t \nabla_{W_2} \cG(\mu_t))=0,
\]
for a matrix-valued function $\bP_t(x): \R^d \to \R^{d \times d}$, where we assume that $\bP_t$ is low-rank and positive semidefinite for all $t$ and $x$. In the following, we will suppress the dependence of $\bP_t$ on $x$. For any $\cG$, we have
\[
    \partial_t  \cG(\mu_t) = -\langle \nabla_{W_2} \cG(\mu_t), \bP_t \nabla_{W_2} \cG(\mu_t) \rangle_{\mu_t}.
\]
Since $\bP_t$ is positive semidefinite, this is nonpositive, and by Lyapunov arguments one can argue that $\lim \inf_{t} \langle \nabla_{W_2} \cG(\mu_t), \bP_t \nabla_{W_2} \cG(\mu_t) \rangle_{\mu_t} = 0$ \cite{polyak2017lyapunov}.

Using the PL inequality, we can derive a more quantitative convergence bound. 
\begin{align}\label{eq:subplineq}
    \partial_t   (\cG(\mu_t) -\cG(\mu_\star))&= -\langle \nabla_{W_2} \cG(\mu_t), \bP_t \nabla_{W_2} \cG(\mu_t) \rangle_{\mu_t} \\ \nonumber
    &= -\frac{\langle \nabla_{W_2} \cG(\mu_t), \bP_t \nabla_{W_2} \cG(\mu_t) \rangle_{\mu_t}}{\|\nabla_{W_2}\cG(\mu_t)\|_{\mu_t}^2}\|\nabla_{W_2}\cG(\mu_t)\|_{\mu_t}^2 \\ \nonumber
    &\leq -2\alpha\gamma_t (\cG(\mu_t) -\cG(\mu_\star)),
\end{align}
where $\gamma_t$ is a continuous function such that
\begin{align}
    0 \leq \gamma_t = \frac{\langle \nabla_{W_2} \cG(\mu_t), \bP_t \nabla_{W_2} \cG(\mu_t) \rangle_{\mu_t}}{\|\nabla_{W_2}\cG(\mu_t)\|_{\mu_t}^2}.
\end{align}
It can be thought of as the amount of alignment of $\nabla_{W_2} \cG(\mu_t)(\cdot)$ with the low-rank function $\bP_t(\cdot)$. We apply Gr\"onwall's inequality to \eqref{eq:subplineq}
\begin{equation}\label{eq:wassprojgronwall}
    \cG(\mu_t) -\cG(\mu_\star) \leq \exp(-2\alpha\int_0^t \gamma_s ds) (\cG(\mu_0) - \cG(\mu_\star)).
\end{equation} 
Therefore, we can feasibly get exponential convergence if $\bP_t$ is well-adapted so that $\gamma_t$ is uniformly bounded below. It is an interesting question for future work to examine such flows in more detail for various functions $\cG$. However, for our work, we will focus on the specific case of the KL  divergence, $\cG = \cF$, as defined in \eqref{eq:kl}. The following lemma establishes geodesic convexity properties of the KL divergence. For a detailed discussion, see \cite{chewi2024log}. 
\begin{lem}
    The function $\cF$ is $\alpha$-strongly geodesically convex iff $V$ is $\alpha$-strongly convex. As a result, if $V$ is $\alpha$-strongly convex, then $\cF$ satisfies the PL inequality \eqref{eq:wasspl}, and the Wasserstein gradient flow exhibits exponential convergence. Furthermore, if $\bP_t$ is continuous, the projected gradient flow $\partial_t \mu_t - \nabla \cdot (\mu_t \bP_t \nabla_{W_2} \cF(\mu_t))=0$ exhibits the convergence bound \eqref{eq:wassprojgronwall}.
\end{lem}

%%%%HERE

\subsection{Subspace Langevin Monte Carlo}
\label{subsec:slmc}

In the previous section, we discussed PLMC, which generalizes LMC and is a special case of mirror LMC. A disadvantage of PLMC is that we must multiply a $d \times d$ positive definite matrix $\bA_k$ by the gradient, which makes the per-iteration complexity and memory significant in high dimensions. 

To make this method more efficient, we consider a subspace version in Wasserstein space, analogous to the subspace gradient flow and gradient descent in Euclidean space. 
For a sequence of low-rank matrices $\bP_k$, the discrete-time process approximating \eqref{eq:plmc} is
\begin{equation}\label{eq:slmc}
    X_{k+1} = X_k -  h_k\bP_k \nabla V( X_k) + \sqrt{2h_k} \bP_k^{1/2} \xi_k,
\end{equation}
which we call the \emph{Subspace Langevin Monte Carlo} algorithm ({SLMC}). For our later analysis, we assume that $\bP_k$ is a random low-rank approximation of $\bA_k$. In particular, we assume that $\bP_k$ is a random eigenblock of $\bA_k$.
\begin{defn}\label{def:eigenblock}
    We say that $\bP_k$ is a rank $r$ eigenblock of $\bA_k$ if it can be written as 
    \[
        \bP_k = \bW_{ik} \bD_{ik} \bW_{ik}^T,
    \]
    where $\bW_{ik} \in \R^{d \times r}$ contains distinct eigenvectors of $\bA_k$ corresponding to the eigenvalues along the diagonal of $\bD_{ik} \in \R^{r \times r}$. In particular, this means we can write the eigenvalue decomposition of $\bA_k = \bW_k \bD_k \bW_k^T$ as
    \[
    \bA_k = \begin{bmatrix}
        \bW_{1k} & \cdots & \bW_{\lceil d/r \rceil k}
    \end{bmatrix}
    \begin{bmatrix}
        \bD_{1k} && \\
        &\cdots & \\
        && \bD_{(d/r)k}
    \end{bmatrix}
    \begin{bmatrix}
        \bW_{1k} & \cdots & \bW_{\lceil d/r \rceil k}
    \end{bmatrix}^T.
    \]
    The last block is allowed to be of variable rank when $r$ does not divide $d$, but for simplicity, we will assume that $d \mod r = 0$
\end{defn}

When $\bA_k$ is the identity and we assume that $\bP_k$ has one nonzero entry along the diagonal, we can recover RCLMC \cite{ding2021random}. Another specific case of SLMC is Block-Coordinate Langevin Monte Carlo, where $\bP_k$ is a projection onto a block of coordinates. To our knowledge, this method has not yet been theoretically studied and is a special case of our later analysis. We emphasize that using a general projection structure allows one to update with respect to a block of coordinates in a new basis at each iteration, which provides flexibility in adapting to the geometry of the target distribution.

\section{Theoretical Analysis of Subspace Langevin Monte Carlo}
\label{sec:theory}

In this section, we conduct a discrete-time analysis of SLMC. Our analysis of SLMC is inspired by previous analyses of mirror LMC and RCLMC \cite{ahn2021efficient,ding2021random}. 

We note that the primary goal of this section is to have a clear picture of the computational complexity of the discussed methods. We measure the complexity in terms of calls to a directional derivative oracle. Therefore, for example, computing the gradient of a function $V: \R^d \to \R$ requires $d$ calls to this oracle, while computing the derivative along $r$ directions requires $r$ calls. This measure of computational cost will be central to our comparison of SLMC with existing methods.

In the following sections, we will give our main theoretical result for SLMC and put it into a broader context. We begin in Section \ref{subsec:coreassump} by discussing the core assumptions shared in the analysis of PLMC and SLMC. Then, in Section \ref{subsec:plm_conv}, we specialize the result of \cite{ahn2021efficient} to give a complexity bound for PLMC. Section \ref{subsec:slmc_conv} presents our convergence result for SLMC, which requires two additional assumptions. We finish in Section \ref{subsec:discussion} by comparing the theoretical complexity of SLMC to LMC, RCLMC, and PLMC.

\subsection{Core Assumptions}
\label{subsec:coreassump}

We begin by outlining some core assumptions shared in the analysis of PLMC and SLMC. 

Our analysis relies on the following notion of relative strong convexity and smoothness.
\begin{defn}\label{def:relscsm}
  The potential $V$ is $m$-relatively strongly convex and $M$-relatively smooth with respect to $\|\cdot\|_{\bB}$ if
  \begin{equation}\label{eq:relcvxsm}
    V(x) + \langle \nabla V(x), y-x\rangle + \frac{m}{2} \|y-x\|_{\bB}^2 \leq V(y) \leq V(x) + \langle \nabla V(x), y-x\rangle + \frac{M}{2} \|y-x\|_{\bB}^2.
  \end{equation}
\end{defn}
When $V$ is twice differentiable, this is equivalent to the condition $m \bB \preceq \nabla^2 V \preceq M \bB$, and these are special cases of the conditions discussed in \cite{lu2018relatively}. We make the following assumption on our sequence of preconditioners $\bA_k$.
\begin{assump}\label{assump:relscsm}
  For all $k \in \N$, the potential $V$ is $m$ relatively strongly convex and $M$ relatively smooth with respect to $\|\cdot\|_{\bA_k^{-1}}$.
\end{assump}

The following remark discusses how relative strong convexity and smoothness can improve the conditioning in Gaussian sampling.
\begin{rem}\label{rem:gaussrel}
Here, the case where $\pi$ is a Gaussian distribution and $V$ is quadratic is illuminating. If $\bA_{k} = \bI$, then Assumption \ref{assump:relscsm} holds when $V$ is $m$-strongly convex and $M$-smooth. Now suppose $V(x) = \frac{1}{2} x^T \bSigma^{-1} x$, in which case $\pi$ is a centered Gaussian distribution. If we take $\bA_k = \bSigma$, then, since $\nabla^2 V = \bSigma^{-1}$, Assumption \ref{assump:relscsm} trivially holds with $m=M=1$. This example illustrates how appropriately adapting the preconditioner can completely alleviate the poor conditioning of $V$, resulting in an optimal relative condition number.
\end{rem}

Since we allow for a time-varying preconditioner $\bA_k$, we assume it is bounded and does not change too quickly. 
\begin{assump}\label{assump:precondchange}
  For the sequence of measures $\mu_{k} = \Law(X_k)$ defined in \eqref{eq:slmc}, we assume that the sequence of $\bA_k$ is chosen so that
    \begin{align}
        W_{2, \bA_k^{-1}}^2(\mu_{k}, \pi) \leq W_{2, \bA_{k-1}^{-1}}^2(\mu_{k}, \pi) + O(h_k^2),
    \end{align}
    where $h_k$ is the step size in \eqref{eq:slmc}.
    We further assume that $\bA_k \preceq C \bI$ for all $k \in \N$.
\end{assump}
This bound states that the $W_{2,\bA_{\cdot}}$ distance between $\mu_k$ and $\pi$ cannot change that much when the preconditioner is changed.

Since using a spatially dependent preconditioner requires an extra correction \cite{ma2015complete} and creates extra dependencies in our analysis, we must assume how the preconditioners $\bA_k$ can depend on the particles $X_1, \dots, X_k$.
\begin{assump}\label{assump:preconddep}
For a sequence of preconditioners $\bA_k$ corresponding to the Markov chain $X_k$, for $k = 1, 2, \dots$, $\bA_j$ is allowed to depend on the distributions of the particles, $\Law(X_i)$, as well as the distributions of their derivatives, $\Law(\nabla^j V(X_i))$, for $i=1, \dots, k$, $j=1, \dots$.
\end{assump}

\subsection{Convergence of PLMC}
\label{subsec:plm_conv}

As a warmup, we recall a proof of convergence for PLMC that follows the arguments given in \cite{ahn2021efficient}. In this work, the authors prove convergence of mirror LMC under the assumptions of relative strong convexity and smoothness.

We begin with a simple lemma that shows the exponential contraction of \eqref{eq:pld} for a fixed preconditioner.
\begin{lem}\label{lem:contract}
    Let $Z_t$ and $Z_t'$ be two copies of the diffusion \eqref{eq:pld} for a fixed $\bA_t = \bA$ as well as the same Brownian motion. Assume that $V$ and $m$-relatively strong convex with respect to $\|\cdot\|_{\bA^{-1}}$. Then, the following contraction holds
    \[
    \|Z_t - Z_t'\|_{\bA^{-1}}^2 \leq \exp(- m t)  \|Z_0 - Z_0'\|_{\bA^{-1}}^2.
    \]
\end{lem}
\begin{rem}
Note that if we optimally couple $Z_0$ and $Z_0'$, this becomes an exponential contraction in the $W_{2,\bA^{-1}}$ distance. Since the Wasserstein distance is an expectation, we could work with a weaker condition than that in Assumption \ref{assump:relscsm} by examining the proof of Lemma \ref{lem:contract}. In particular, we could instead require that relative strong monotonicity of $\nabla V$ (a consequence of relative strong convexity) holds in expectation along the flows $\Law(Z_t)$ and $\Law(Z_t')$. An interesting line of future work could study this assumption in more detail.
\end{rem}

%%%%

While this lemma can prove continuous time convergence of the preconditioned Langevin diffusion \eqref{eq:pld} for a fixed preconditioner, it is also an essential piece of the proof of convergence of PLMC. The following theorem states a convergence bound for PLMC. Its proof mostly follows Theorem 2 of \cite{ahn2021efficient}, with the addition of the varying preconditioner $\bA_k$. We include a simplified proof of this theorem in Appendix \ref{app:plmc_conv}.
\begin{thm}\label{thm:plmc_conv}
    Suppose that Assumptions \ref{assump:relscsm}, \ref{assump:precondchange}, and \ref{assump:preconddep} hold, and PLMC is run with constant step size $h_k = h$. Then, PLMC achieves the error bound
    \begin{equation}
        W_{2, \bA_N^{-1}}(\mu_{N}, \pi) \leq (1-m h)^{N/2} W_{2, \bA_0^{-1}}(\mu_{0}, \pi) + O( \sqrt{\frac{M}{m} d h}).
    \end{equation}
\end{thm}

It is possible to translate from $W_{2, \bA_N^{-1}}$ bounds to $W_2$ bounds at the cost of a factor $\|\bA_N\|_2$ (which is bounded by a constant from Assumption \ref{assump:precondchange}). By choosing an appropriate $h$, we find the total complexity in terms of directional derivative computations to achieve $\epsilon$ error to be
\begin{equation}\label{eq:itcompplmc}
    N_{\mathrm{PLMC}} = O\left(\frac{d^2 \kappa_{\rel}}{\epsilon^2 m } \log \frac{W_{2,\bA_0^{-1}}(\mu_0, \pi) }{\epsilon} \right),
\end{equation}
where $\kappa_{\rel} = M/m$ is the relative condition number.
Again, the extra factor of $d$ comes from needing to compute $d$-directional derivatives at each iteration. Each iteration also requires multiplying the gradient by a $d \times d$ matrix, which incurs an additional $d^2$ complexity.

\subsection{Convergence of SLMC}
\label{subsec:slmc_conv}

We now present our main theoretical convergence result for SLMC. It relies on the following assumptions. First, we have an assumption on how the matrices $\bP_k$ are generated.
\begin{assump}\label{assump:proj}
Following Definition \ref{def:eigenblock}, we assume the preconditioner $\bA_k$ is partitioned into eigenblocks $\bW_{ik} \bD_{ik} \bW_{ik}^T$, for $i=1, \dots, \lceil d/r \rceil$. At iteration $k$, we choose an index $i$ with probability $\phi_{ik}$ and set $\bP_k = \bW_{ik} \bD_{ik} \bW_{ik}^T$ and $h_k = h/\phi_{ik}$, where $h>0$ is fixed. We thus see that $h\bA_k = \E h_k \bP_k$.
\end{assump}
Note that this procedure is preconditioned block coordinate descent after applying a rotation $\bW_k^T$. In the case of $\bW_k =\bD_k = \bI$ and $r=1$, this is just the standard RCLMC \cite{ding2021random}.

Our final assumption allows smoothness to adapt to the update directions chosen at each iteration. 
\begin{assump}\label{assump:sm}
    The function $V$ is $M_k(\bU)$ directionally smooth with respect to $\bA_k$ along any linear subspace spanned by $\bU \in O(d,r)$. That is, for all $y = x + \bU \bU^T \delta$, 
    \[
        \|\bU\bU^T (\nabla V(x) - \nabla V(y))\|_{\bA_k} \leq M_k(\bU) \|\bU \bU^T(x-y)\|_{\bA_k^{-1}}.
    \] 
    We further assume that 
    \[
    \|\nabla V(x) - \nabla V(y))\|_{\bA_k} \leq M_k \|x-y\|_{\bA_k^{-1}}, \ \forall \ x,y \in \R^d,
    \]
    which then implies the upper bound in \eqref{eq:relcvxsm}. It is not hard to show that $M_k(\bU) \leq M_k$ for all $\bU \in O(d,r)$.
\end{assump}

% \begin{assump}\label{assump:bdstep}
%     At iteration $k$, the step size $h_k$ is chosen corresponding to the picked block, i.e., $h_k = h_{ik}$ for some $i=1, \dots, d/r$. We assume that $h=\frac{h_{i}}{\phi_{ik}}$ is fixed.
% \end{assump}

With this final assumption, we are now ready to state our main theorem.
\begin{thm}\label{thm:conv}
    Suppose Assumptions \ref{assump:relscsm}, \ref{assump:precondchange}, \ref{assump:preconddep}, \ref{assump:proj}, and \ref{assump:sm} hold, and additionally that $h \leq \min_i \phi_{ik}/M_k$ for all $k$. Then we have the error bound
    \begin{align*}
        W_{2,\bA_N^{-1}}(\mu_N, \pi) &\lesssim \exp(-\frac{h m N}{4}) W_{2,\bA_0^{-1}}(\mu_0, \pi) +  \sqrt{\frac{1}{m}\sum_{j=1}^N (1-\frac{h m}{2})^{N-j} \sum_{i=1}^{d/r} [  \frac{r  h^2 M_{ij}^2}{\phi_{ij}}] },
    \end{align*}
    where $M_{ij} = M_j(\bW_{ij})$ is the directional smoothness corresponding to the $i$th eigenblock at iteration $j$.
\end{thm}
The proof of this theorem is given in Appendix \ref{app:thmproof} and is similar to that of the main theorem of \cite{ding2021random} with a few twists. The proof follows a Wasserstein coupling style argument developed by \cite{dalalyan2019user}. The last term depends on how well the sequence of preconditioners can adapt to the relative smoothness condition over time.

We can simplify the above bound to get a clearer picture of the complexity of SLMC.
If $j$ is the index that maximizes $ \sum_i [  \frac{M_{ij}^2}{\phi_{ij}}] $ and we write $M_{ij} = M_i$, $\phi_{ij} = \phi_i$, $\kappa_{\rel,i} = M_i/m$, and $h \asymp \frac{\epsilon^2}{r \sum_i  \kappa_{\rel,i}^2/\phi_{i}}$ in Theorem \ref{thm:conv}, we see that the complexity of SLMC is
\begin{equation}\label{eq:itcompslmc}
    N_{\mathrm{SLMC}} = O\left( \frac{r^2  \sum_{i=1}^{d/r}  \kappa_{\rel,i}^2/\phi_{i}}{\epsilon^2 m} \log \frac{W_2(\mu_0, \pi) }{\epsilon} \right).
\end{equation}
Note that there is an extra factor of $r$ rather than $d$ since we do not need to compute all directional derivatives at each iteration. We discuss and further compare all of these results in the next section. 

\begin{rem}
    Due to Assumption \ref{assump:preconddep}, we can allow the preconditioners to depend on the laws of the particles. To implement this in practice, we could compute preconditioners from systems of independent particles at each iteration, which would approximate preconditioners depending on $\mu_k = \Law(X_k)$. However, our analysis does not apply to subspace versions of adaptive algorithms like Adagrad Langevin, RMSProp Langevin, and Riemannian Langevin methods \cite{yu2024scalable}. This is because, in these methods, the preconditioner depends on the spatial variable $X$ itself rather than its law. It would be interesting to analyze subspace versions of these methods in future work.
\end{rem}

\subsection{Discussion}
\label{subsec:discussion}

In this section, we present a comparison of the various methods. We will compare SLMC to LMC, PLMC, and the RCLMC method of \cite{ding2021langevin}.

We first recall the complexity result of \cite{ding2021langevin}. Let $\beta_i$ denote the smoothness in the $i$th coordinate direction, as in \cite{ding2021random}, and suppose that $V$ is $\alpha$-strongly convex. Denote $ \kappa_i = \beta_i / \alpha$ as the condition number along the $i$th direction with respect to the strong convexity parameter $\alpha$, and $ \kappa = \beta / \alpha$.
The complexity of RCLMC, in this case, is
\[
     N_{\sf RCLMC} = O\left( \frac{\sum_{i=1}^d \kappa_i^2/\phi_{i}}{\epsilon^2 \alpha} \log \frac{W_2(\mu_0, \pi)}{\epsilon} \right).
\]

Let us first compare SLMC and LMC. The most straightforward case occurs when $\bA_k = \bI$ for SLMC, and we use uniform random sampling, $\phi_{ik} = r/d$. In this case, relative strong convexity and smoothness reduce to normal strong convexity and smoothness, so $m = \alpha$ and $M=\beta$. If we use the worst case bound to set $\kappa_{\rel, i} = \kappa_i = \kappa$, then we see that SLMC is worse than LMC by a factor of $\kappa$. 
On the other hand, if we use general $\bA_k$ and take $\phi_{ik} = \kappa_{\rel, i} /\sum_j \kappa_{\rel, j}$, then we find that SLMC outperforms LMC once $r^2 \frac{(\sum_{i=1}^{d/r} \kappa_{\rel, i})^2}{m} \leq \frac{d^2 \kappa}{\alpha}$.
When the potential $V$ is highly skewed and $\bA_k$ is well-adapted to it ($\kappa_{\rel, i} \ll \kappa$ and $m \gg \alpha$), we see that there are definitive gains in performance for SLMC. 

We next focus on comparing SLMC and PLMC. This comparison shows that SLMC has the best bound out of existing processes in specific ill-conditioned examples. In particular, for SLMC to beat PLMC, we need
\[
   r (\sum_{i=1}^{d/r} \kappa_{\rel,i})^2\geq  d \kappa_{\rel}.
\]
Therefore, SLMC improves over PLMC when the directional conditioning bounds are better than the uniform conditioning bounds required by PLMC.

Finally, to compare SLMC to RCLMC, we note that our result is a strict generalization of RCLMC's. In particular, we recover their result as a subcase of ours but extend it to arbitrary block size and relative conditioning.

We summarize these complexity results in Table \ref{tab:comp}. As we can see, the results of different methods may be optimal in various settings.
\begin{table}[ht]
\centering
\begin{tabular}{|c| c c c c|} 
 \hline
   & LMC & PLMC & RCLMC & SLMC \\ \hline 
  Conditioning  & standard & relative & directional standard & directional relative \\ \hline
  Non-random & $\tilde O(\frac{d^2  \kappa}{\epsilon^2 \alpha } )$ &$\tilde O(\frac{d^2 \kappa_{\rel}}{\epsilon^2 m } )$ & -  & -\\ \hline 
  $\phi_{ik} = r/d$ &  -  & -  & $\tilde O(\frac{d \sum_{i=1}^{d}  \kappa_i^2}{\epsilon^2 \alpha} )$ & $\tilde O(\frac{dr \sum_{i=1}^{d/r} \kappa_{\rel,i}^2}{\epsilon^2 m} )$ \\ \hline 
  $ \phi_{ik} = M_{ik}/\sum_i M_{ik}$ & -  & - & $\tilde O(\frac{d (\sum_{i=1}^{d}  \kappa_i)^2}{\epsilon^2 \alpha} )$ & $\tilde O(\frac{dr (\sum_{i=1}^{d/r} \kappa_{\rel,i})^2}{\epsilon^2 m} )$  \\ \hline
\end{tabular}\vspace{.1cm}
\caption{Table comparing complexity bounds for LMC, PLMC, RCLMC, and SLMC measured in terms of total number of directional derivative computations. Our analysis of SLMC is strictly more general than that of RCLMC and allows for a method that can be much more efficient than vanilla LMC. The $\tilde O$ removes constants and dependence log factors, such as $\log W_2(\mu_0, \pi)/\epsilon$.}
\label{tab:comp}
\end{table}

RCLMC and SLMC have the additional benefit of significantly lower memory usage than LMC and PLMC. In particular, these methods do not need to store the full gradient at each iteration. Furthermore, SLMC does not need to store the full $d \times d$ preconditioner at each iteration, requiring $O(d^2)$ memory. This is prohibitive in high dimensions. This is essential if one wants to efficiently scale adaptive Langevin methods to high dimensions, as is done in the recent works on training large language models \cite{feinberg2024sketchy,zhao2024galore,liang2024memory}.

\section{Experiments}
\label{sec:exp}

We perform three experiments that demonstrate the practical efficiency and adaptability of SLMC. The first example examines the choice of fixed preconditioning matrices and rank on an ill-conditioned Gaussian distribution. The second experiment explores the use of subspace approximations to Hessian preconditioners in Bayesian logistic regression. The third and final experiment demonstrates the effectiveness of subspace approximations of adaptive preconditioners on a funnel distribution. 

While our current theory does not cover the last example, it illustrates the effectiveness of our method when incorporated with adaptive preconditioners. We believe studying memory-efficient adaptive sampling methods for large-scale non-log-concave sampling problems is a promising direction for future work.

\subsection{Ill-Conditioned Gaussian Sampling}

In our first experiment, we seek to sample from a specific ill-conditioned Gaussian distribution. This example is inspired by \cite{ding2021langevin} and explores the adaptability of SLMC in terms of choice of preconditioner and rank. In this experiment, we compare different variants of SLMC with LMC, PLMC, and RCLMC. Throughout, we will set up SLMC to have a fixed preconditioner, which, in effect, is a form of block-coordinate LMC. 

The target distribution $\pi$ is a centered Gaussian distribution in dimension $d=20$ with covariance $\bSigma$ defined as follows. For a random $5 \times 5$ Gaussian matrix $\bG$ and a $10 \times 10$ random orthogonal matrix $\bU$, we define a block diagonal structure
\[
    \bSigma^{-1} = \begin{bmatrix}
   \bU\begin{pmatrix}
        (\bG + 10 \bI_{5})(\bG + 10 \bI_{5})^T& \bzero \\ \bzero & \bI_{5}
    \end{pmatrix} \bU^T & \\ & \bI_{10}
\end{bmatrix}.
\]
In the following, we define a test function $\phi(x) = |\bone^T x|$, where $\bone$ is the ones vector, and compute the error as $\mathsf{Err} = \left| \frac{1}{N} \sum_i \phi(X_i) - \E_\pi \phi \right|$.
The initial particle generated as $X_0\sim N(1, \bI_{20})$, and for each method we take $N=20000$ steps. We will measure error versus complexity for all methods, which is measured in directional derivative computations.

Figure \ref{fig:basis} displays various experimental results on sampling from this target distribution.
First, the top left image explores no preconditioning, $\bA_k =\bI$. We see that SLMC can match the performance of LMC when the block size is adapted to the 10-dimensional block structure of $\bSigma^{-1}$ while having a reduced memory footprint. 

Second, on the top right, we explore simple preconditioning. Here, SLMC and PLMC have a diagonal preconditioner that is $\bI$ on the upper left $10 \times 10$ block and $10 \bI$ on the lower $10 \times 10$ block. From this, preconditioning enables fast initial convergence, especially when the block size is adapted to the covariance structure.

Third, in the lower left image of Figure \ref{fig:basis}, we run an experiment to demonstrate how the choice of basis can assist SLMC in converging faster. In particular, the subspaces are eigenspaces of $\bSigma$. Because of this, as is discussed in Remark \ref{rem:gaussrel}, the relative conditioning is much better for SLMC and exhibits faster convergence than LMC and RCLMC. 

Finally, in the last experiment in the lower right of Figure \ref{fig:basis}, we again use the identity preconditioner but change the orthogonal blocks from which we sample the basis for SLMC. SLMC rotated refers to conducting SLMC with blocks taken from $\diag(\bU, \bI_{10})$. Note that in this new basis, $\bSigma^{-1}$ now has an upper left $5 \times 5$ block that is not the identity, while the lower right block is now $\bI_{15}$. With $r=5$, the rotated version of SLMC performs better than the unrotated version since the choice of basis creates a $5$-dimensional block structure.

\begin{figure}
    \centering
    \includegraphics[width=0.4\linewidth]{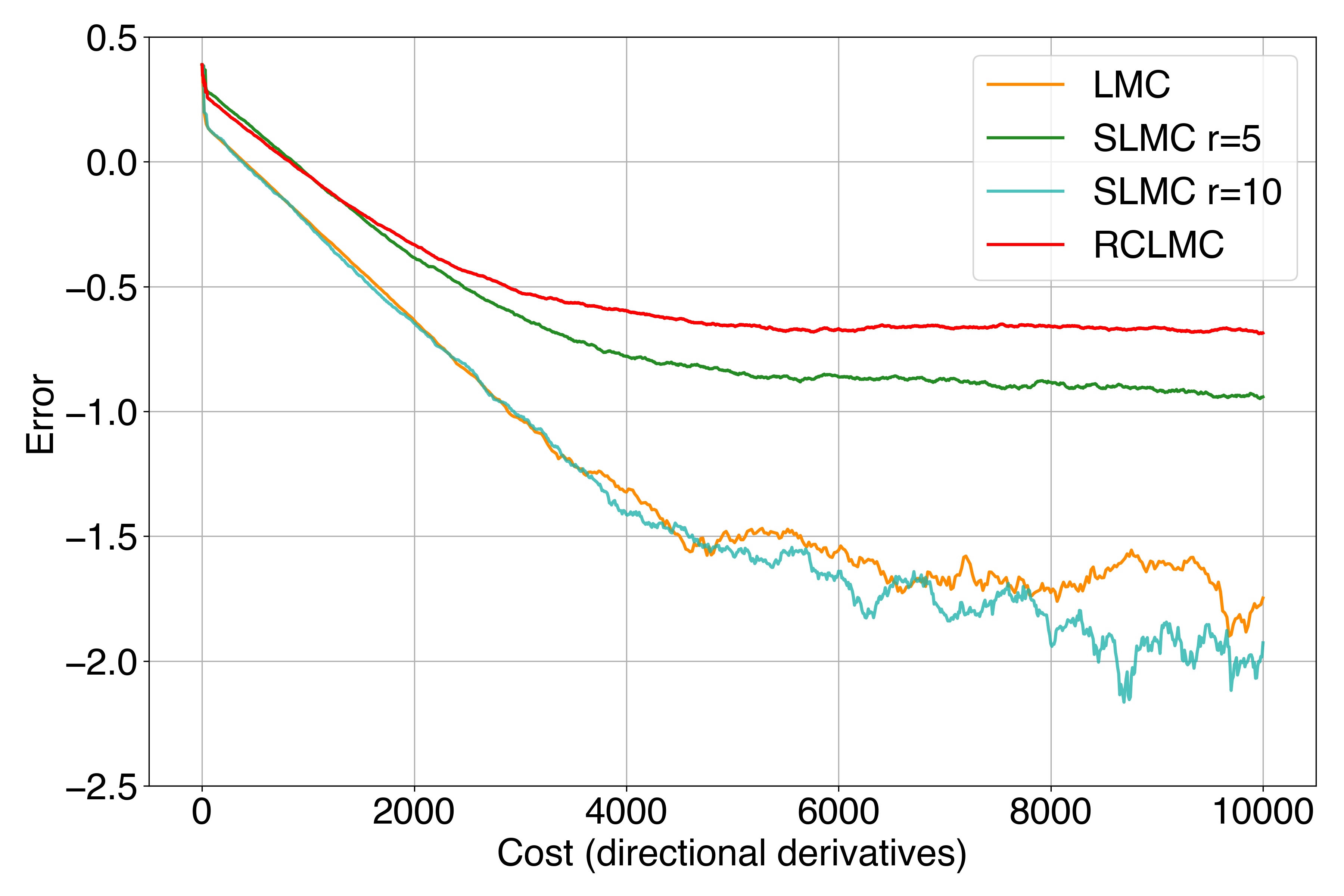}
    \includegraphics[width=0.4\linewidth]{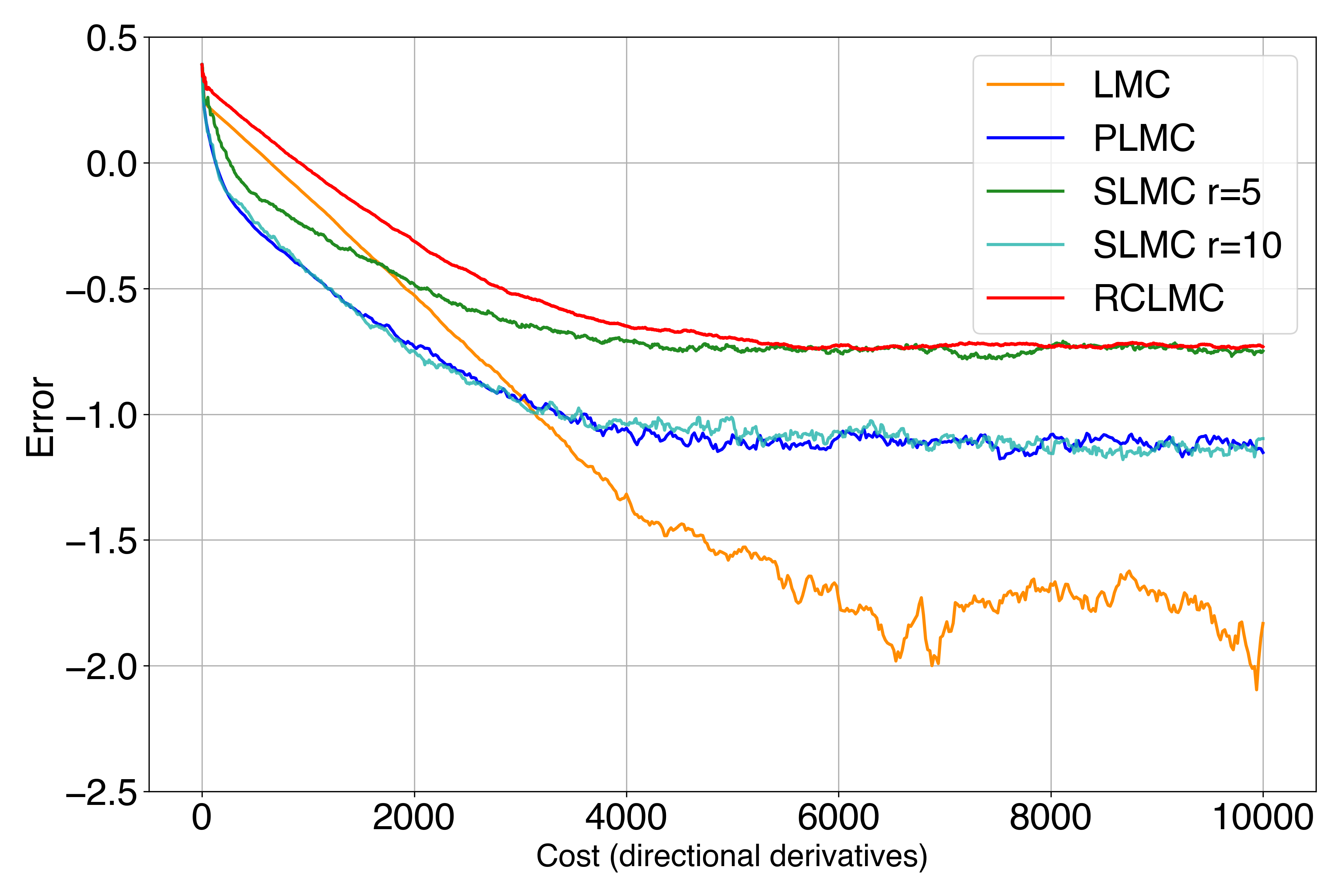}
    \includegraphics[width=0.4\linewidth]{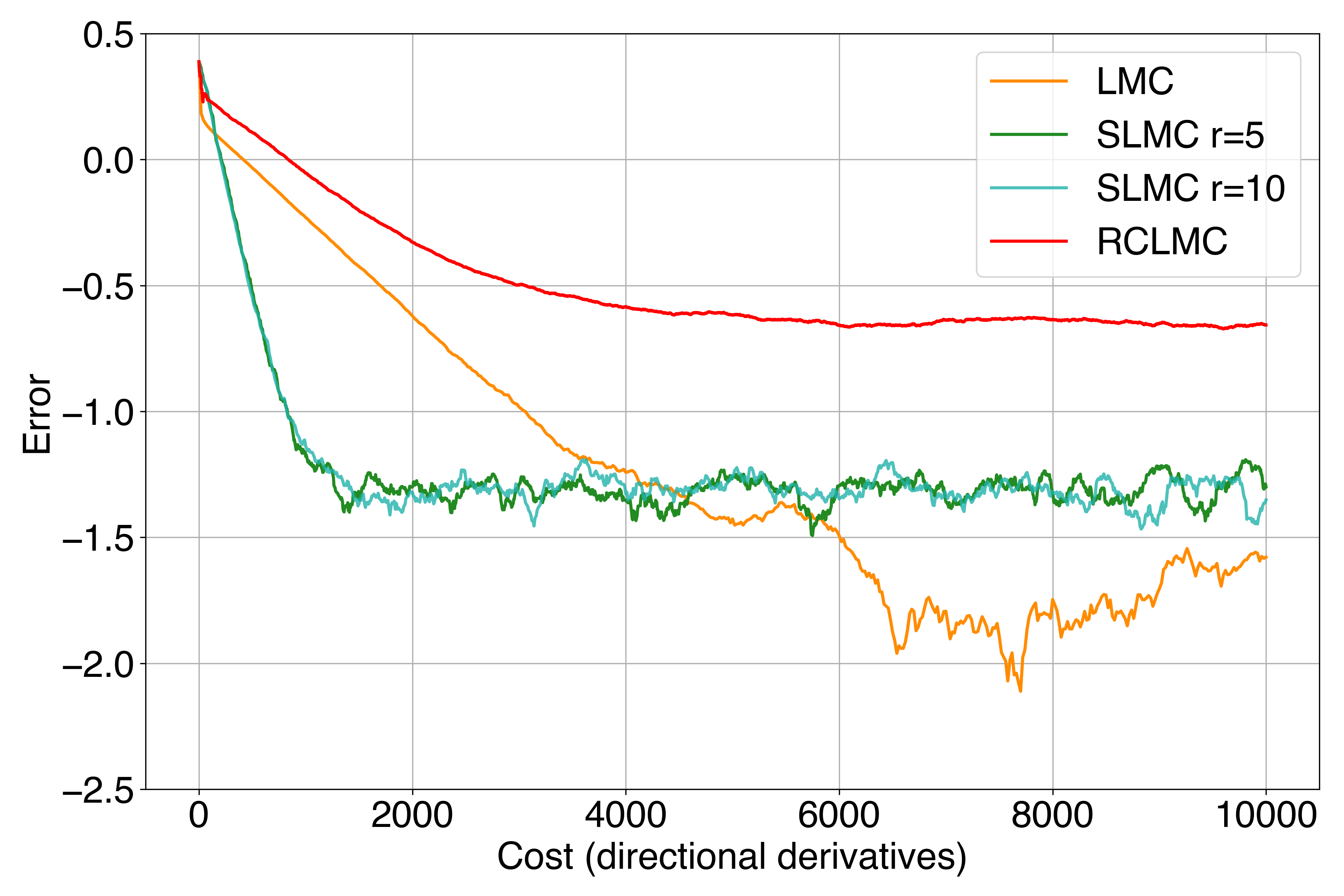}
    \includegraphics[width=0.4\linewidth]{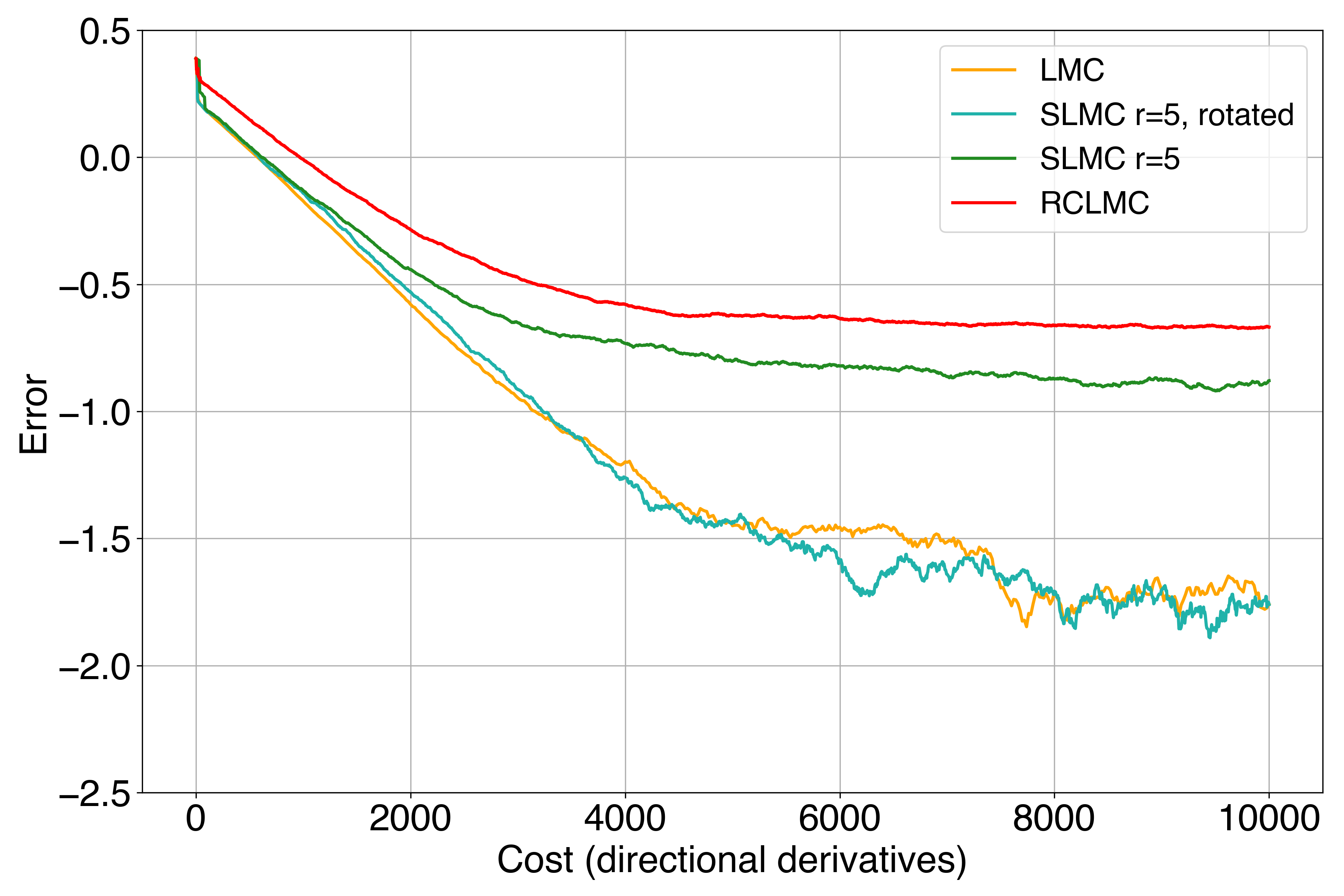}
    \caption{Experiments demonstrating the convergence of SLMC and PLMC compared to LMC and RCLMC for a diagonal preconditioner. \textbf{Top Left:} We set $h=0.01$ and use $\bA_k = \bI$. As we can see, setting the dimension equal to the upper left block allows SLMC to converge as fast as LMC. \textbf{Top Right:} We again set $h=0.01$ and let SLMC and PLMC have a diagonal preconditioner that is $\bI$ on the upper left $10 \times 10$ block and $10 \bI$ on the lower $10 \times 10$ block. As we can see, using a nonuniform step size allows for faster initial convergence since it is adapted to the covariance structure. This step size incurs a larger final bias. \textbf{Bottom Left:} SLMC blocks are taken from the eigenvalue decomposition of the covariance. The step size for SLMC is considered to be larger at $h=0.5$, while for LMC and RCLMC, it is $0.01$. The adaptation allows SLMC to converge rapidly at the onset while having a larger bias due to a larger effective step size. \textbf{Bottom Right:} SLMC now uses blocks from the rotation such that the top left block is $5 \times 5$. As we can see, the SLMC method with $r=5$ can now adapt to the blocks and converge faster than $r=10$.}
    \label{fig:basis}
\end{figure}

\subsection{Bayesian Logistic Regression}

For our next experiment, we explore a problem in Bayesian inference. In our example, we consider sampling from an ill-conditioned posterior in Bayesian logistic regression. 

This classic problem in Bayesian inference relies on observed data $(X_i, y_i)$, $i=1, \dots, n$, where the $X_i$ are the covariates and the $y_i$ are binary response variables. Assuming a Gaussian prior $\theta \sim N(0, \bSigma_{\theta})$, we arrive at the posterior
\[
    \pi(\theta | (X_1, y_1), \dots, (X_n, y_n)) \propto \exp\left[ \frac{1}{2} \theta^T \bSigma_{\theta}^{-1} \theta  + \sum_{i=1}^n y_i \theta^T X_i - \log(1+e^{\theta^T X_i)} \right].
\]
One way to infer properties of this posterior is to generate samples from it and then use these samples to approximate various moments of the posterior.

In our setup, we assume that $\theta$ and the $X_i$ are in $\R^2$. We set $n=100$ and generate i.i.d. data $X_i \sim N(0, \diag(10, .1))$, $y_i \sim \mathsf{Ber}(\mathsf{ilogit}((1, 1) X_i))$, where the logistic function is $\mathsf{ilogit}(z) = e^z/(1+e^z)$. We assume that the prior covariance is $\bSigma_{\theta} = \diag(1, 100)$. 

We run each method's $\ell=100$ independent chains to generate independent samples from the posterior distribution. We initialize the samples as $\theta_i^{0} \sim \mathsf{Unif}(-0.1, 0.1)$.

In this example, we set $r=1$ and we run SLMC with $\bA_k = \bI$ with two step sizes $h=0.1$ and $h=0.01$, as well as SLMC whose projections are generated as eigenblocks of $\bA_k = [\frac{1}{\ell} \sum_{j=1}^\ell \nabla^2 V(\theta_j^k)]^{-1}$. This choice is meant to approximate $[\E_{\theta \sim \mu_k} \nabla^2 V(\theta)]^{-1}$, where $\mu_k$ is the distribution of the random variable $\theta_{\cdot}^k$ after $k$ steps of the Markov chain. Our earlier theory applies to the preconditioner $[\E_{\theta \sim \mu_k} \nabla^2 V(\theta)]^{-1}$ since it is independent of the particles $\theta_i^k$.

Figure \ref{fig:logsamples} gives samples and contours for all samples generated by three SLMC methods. As we can see, at least visually, using the average Hessian preconditioner allows the method to adapt to the ill-conditioned shape of the posterior contours.

\begin{figure}
    \centering
    \includegraphics[width=0.32\linewidth]{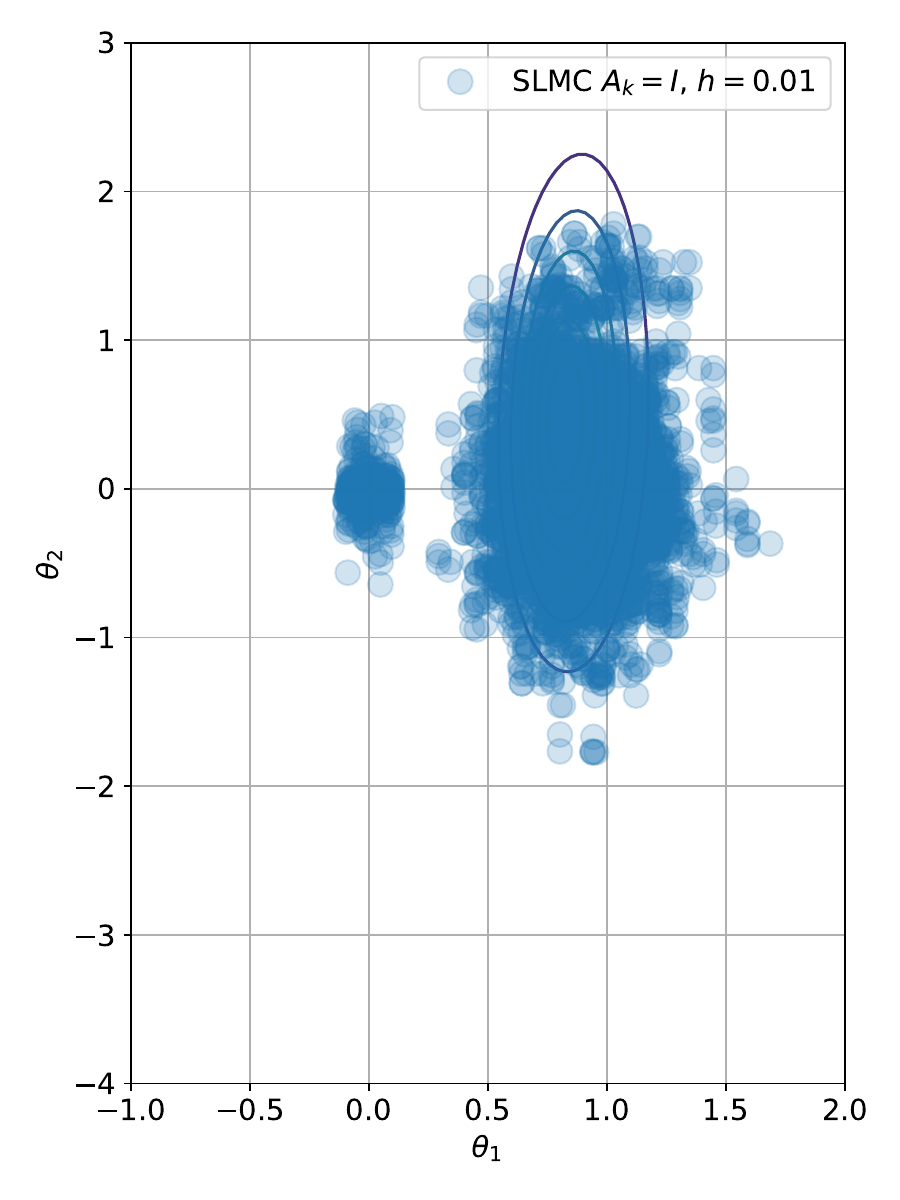}
    \includegraphics[width=0.32\linewidth]{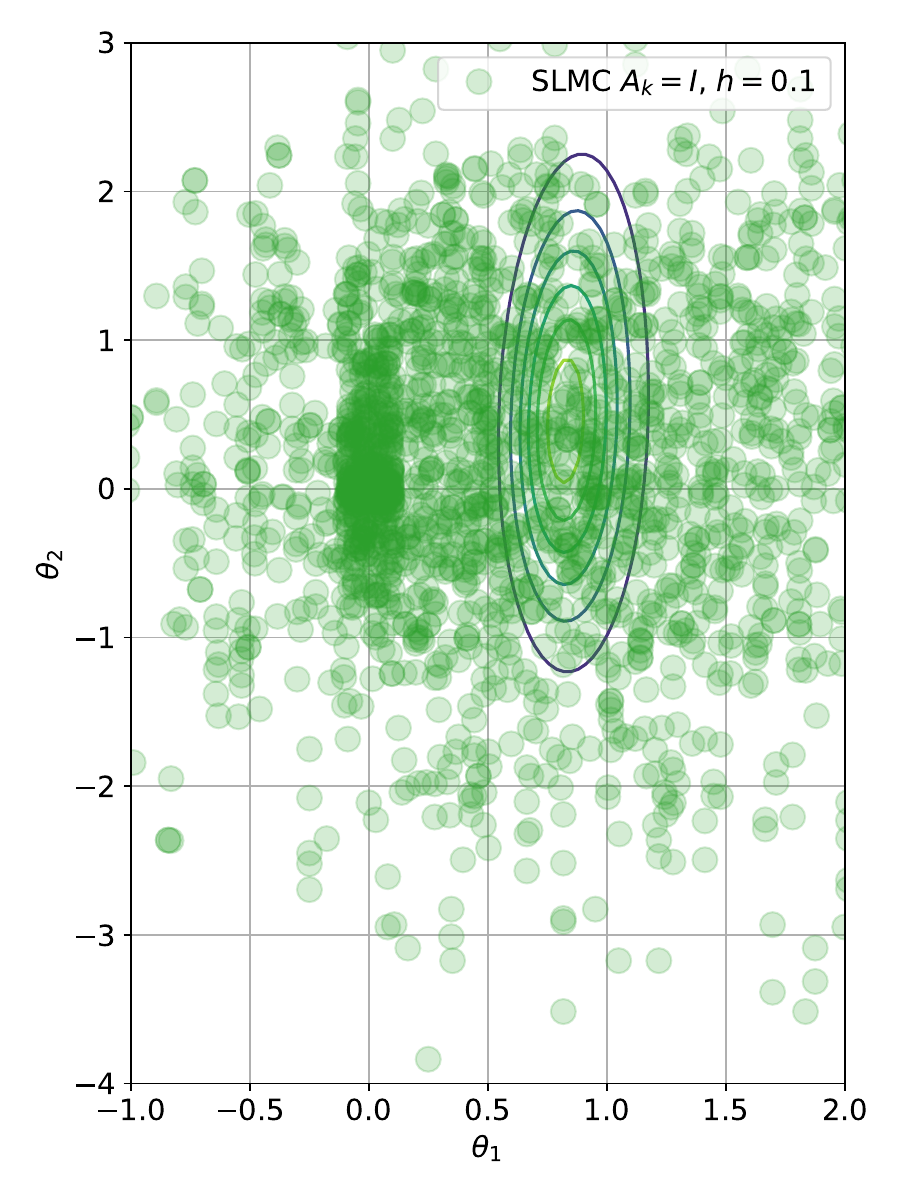}
    \includegraphics[width=0.32\linewidth]{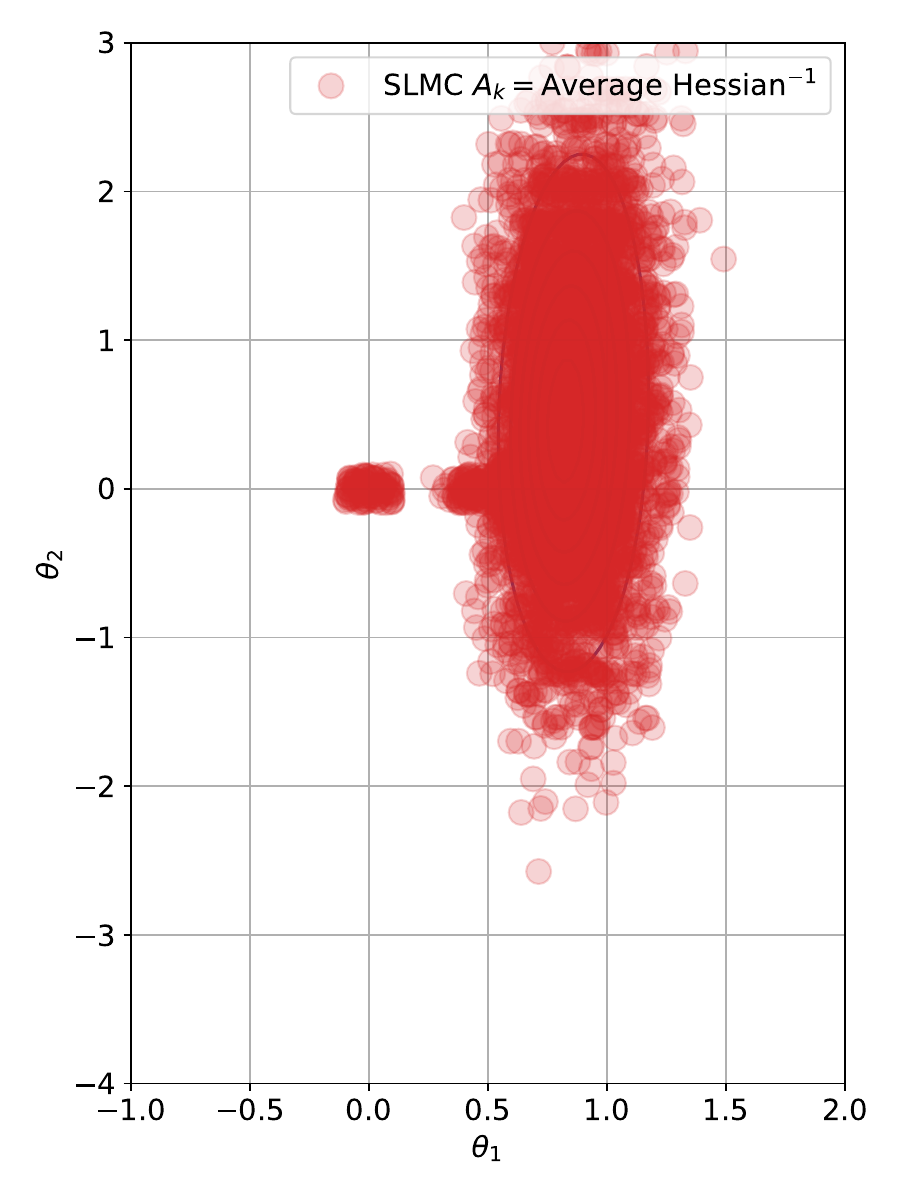}\\
    \includegraphics[width=0.32\linewidth]{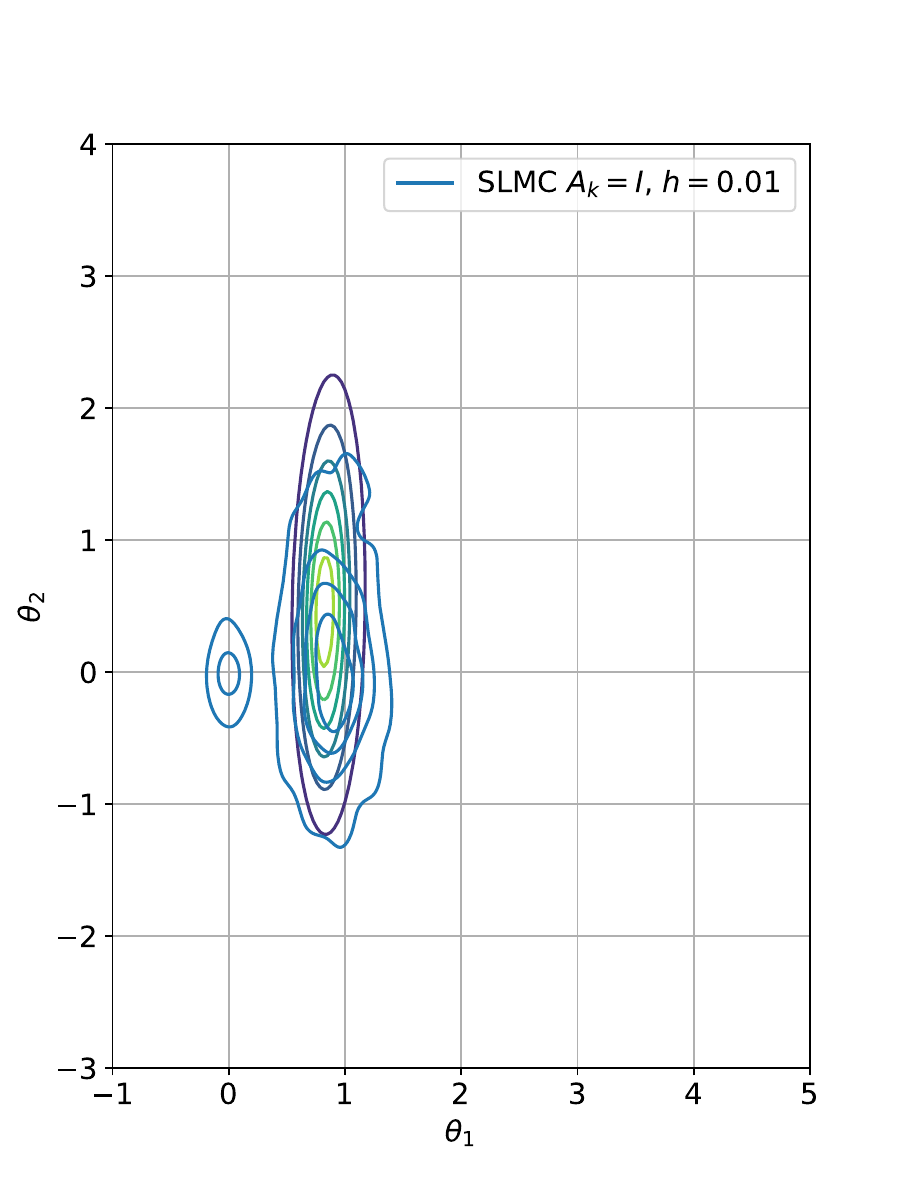}
    \includegraphics[width=0.32\linewidth]{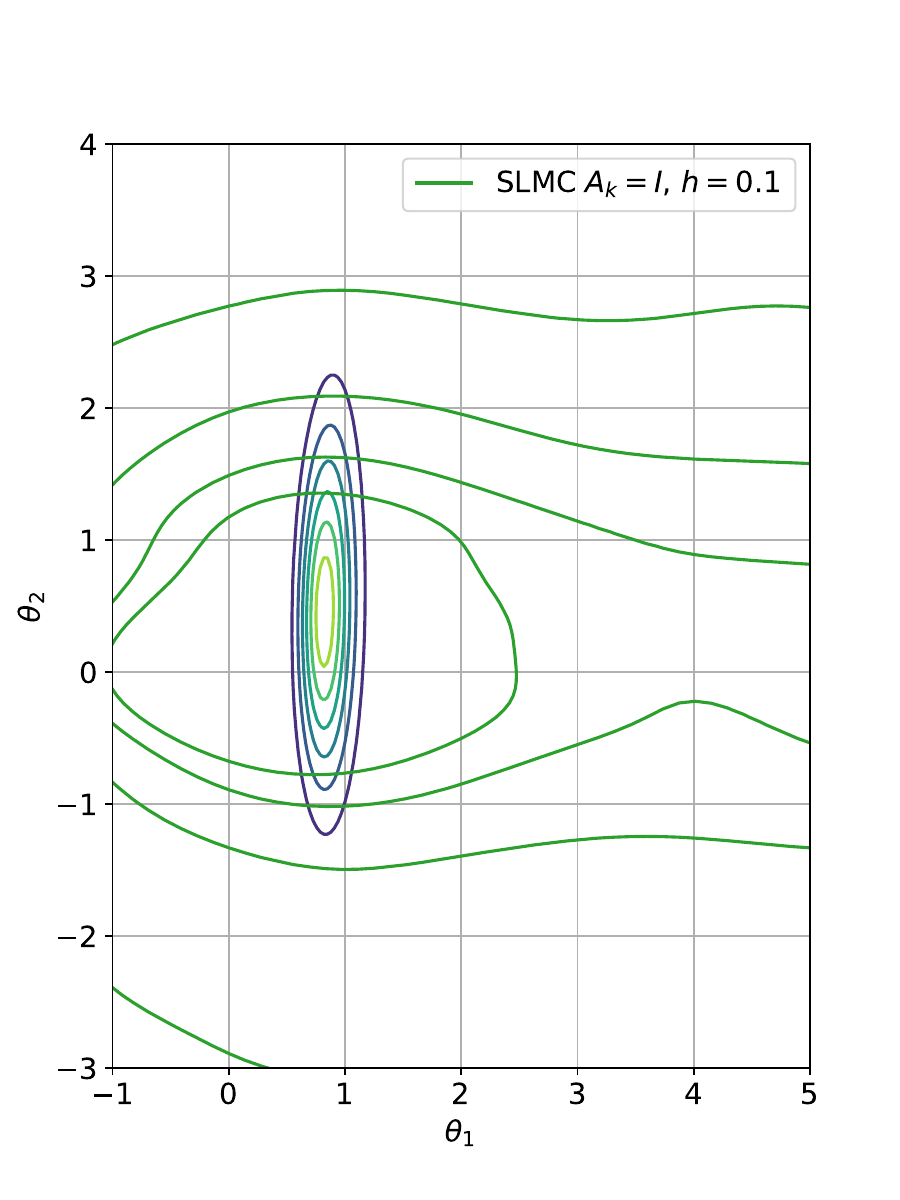}
    \includegraphics[width=0.32\linewidth]{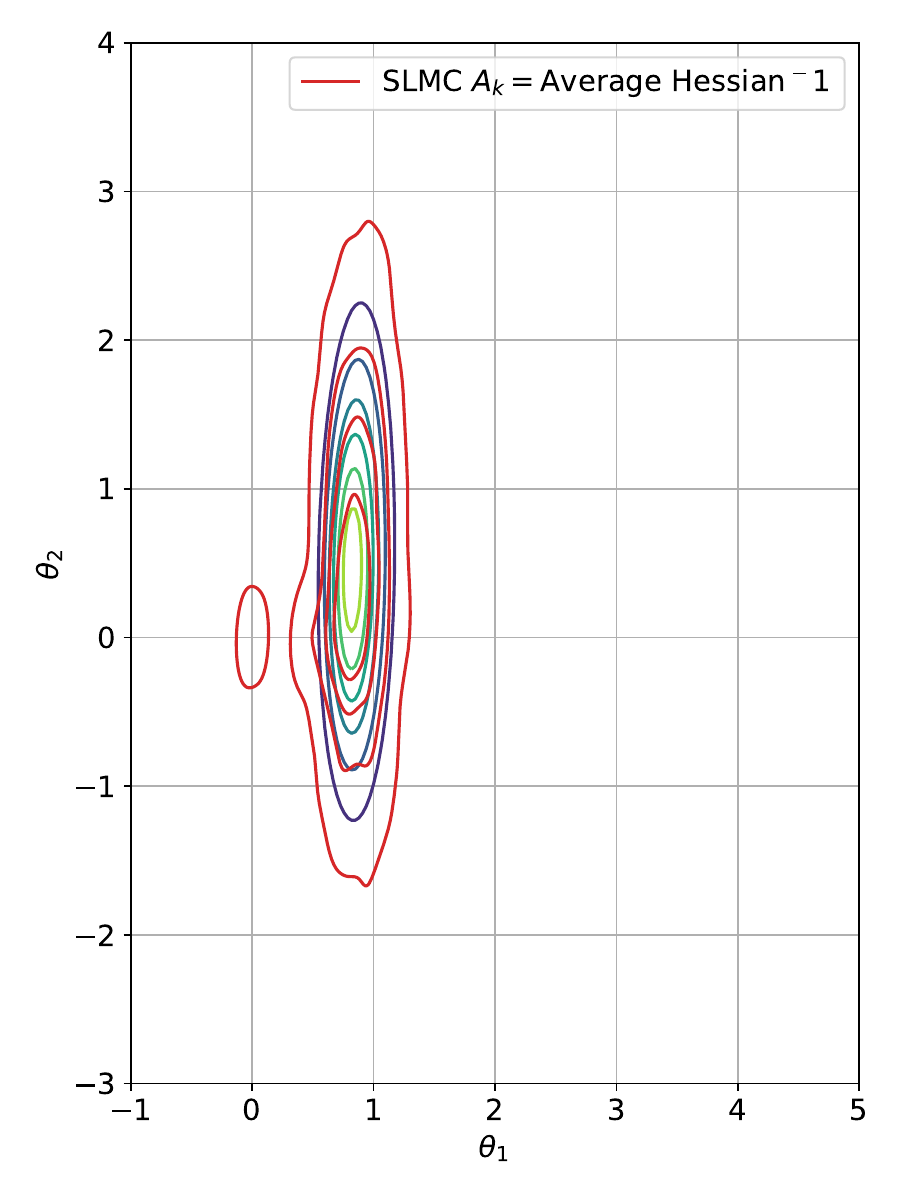}
    \caption{Samples and contours for Bayesian logistic regression experiment. \textbf{Left:} samples and contours when $\bA_k = \bI$ and $h=0.01$. \textbf{Middle:} samples and contours when $\bA_k = \bI$ and $h=0.1$ \textbf{Right:} samples and contours when $\bA_k = [\frac{1}{\ell} \sum_{j=1}^\ell \nabla^2 V(\theta_j^k)]^{-1}$ and $h=0.5$.}
    \label{fig:logsamples}
\end{figure}

To measure the quality of the samples over iterations, we run each method 50 times and compute the kernelized Stein discrepancy \cite{gorham2017measuring}. We use the be inverse multiquadric (IMQ) kernel with $\beta = -1/2$, and choose the bandwidth using the median heuristic as is done in \cite{liu2016stein,maurais2024sampling}. As we can see, the average Hessian preconditioning allows for better accuracy than the identity preconditioning.

\begin{figure}
    \centering
    \includegraphics[width=0.5\linewidth]{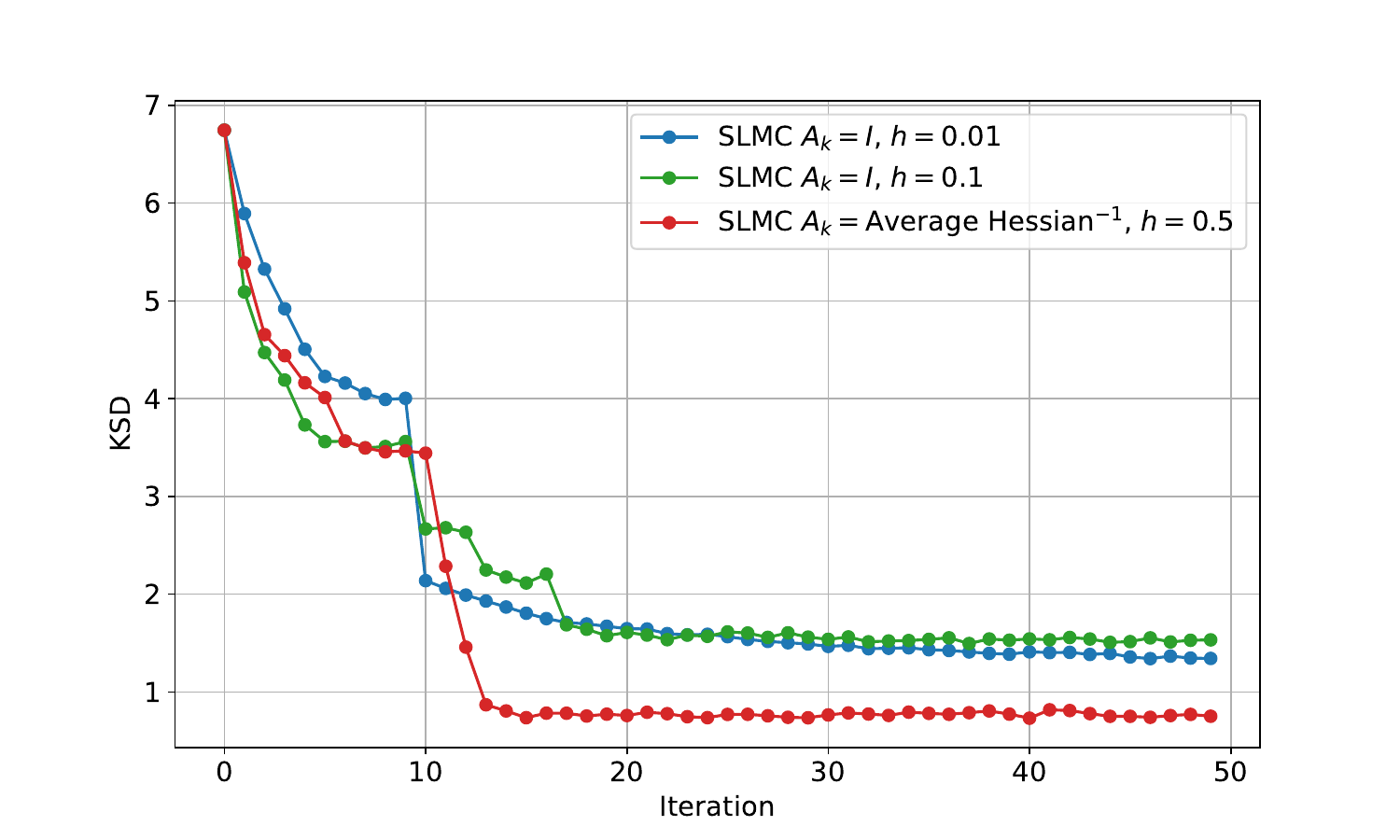}
    \caption{Kernelized Stein Discrepancy versus iterations for samples from Bayesian logistic regression experiment. The discrepancy values are averaged over 20 random generated examples.}
    \label{fig:enter-label}
\end{figure}

\subsection{Adaptive Methods and the Funnel Distribution}

Finally, we show how subspace Langevin methods can be used with adaptive preconditioning to develop efficient methods for sampling from complex distributions. We replicate the experiment from \cite{yu2024scalable} and compare preconditioners computed by RMSProp \cite{tieleman2012rmsprop} and Adagrad \cite{duchi2011adaptive,mcmahan2010adaptive} to the identity preconditioner in two settings. First, we try to sample from the standard funnel distribution with these methods, and then we try to sample from a rotated version of the funnel using these methods. We measure sample quality by qualitatively looking at the marginal along the $y$-axis for the generated samples.

In the following, we assume that we observe the gradients exactly at every iteration.
For the implementation of RMSProp and Adagrad, we initialize them at 0 and iteratively update the preconditioners as
\[
    \bA_k^{\text{RMSProp}} = [.99(\bA_{k-1}^{\text{RMSProp}})^{-2} + .01 \diag(\nabla V(\theta_k)_1^2, \nabla V(\theta_k)_2^2)]^{-1/2},
\]
\[
    (\bA_k^{\text{Adagrad}}) = [.99(\bA_{k-1}^{\text{Adagrad}})^{-2} + .01 \nabla V(\theta_k) \nabla V(\theta_k)^T]^{-1/2},
\]
where $\theta_k$ is the $k$th point in the Markov chain and $\nabla V(\theta_k)_j$ is the $j$th coordinate of the gradient.

Figure \ref{fig:funnel} shows the results of running subspace Langevin with the identity and diagonal RMSProp preconditioners for the standard funnel distribution of \cite{yu2024scalable}. As we can see, the RMSProp better adapts to the geometry of the funnel distribution. Adagrad also has some adaptation, but it is slightly less accurate because the diagonal preconditioning of RMSProp encodes a bias that aligns with the symmetries of the distribution and the marginal we wish to match.

Figure \ref{fig:funnelR} repeats this experiment, except now we rotate the funnel by precomposing it with the rotation
\[
    \bW = \begin{bmatrix}
        \sqrt{3}/2 & 1/2 \\ -1/2 & \sqrt{3}/2
    \end{bmatrix}.
\]
Now we see that since the axes are no longer aligned with the coordinate axes, the Adagrad method adapts better than RMSProp.

Altogether, our work opens the door to exploring memory-efficient adaptive Langevin methods \cite{feinberg2024sketchy,zhao2024galore,liang2024memory}.

\begin{figure}
    \centering
    \includegraphics[width=0.32\linewidth]{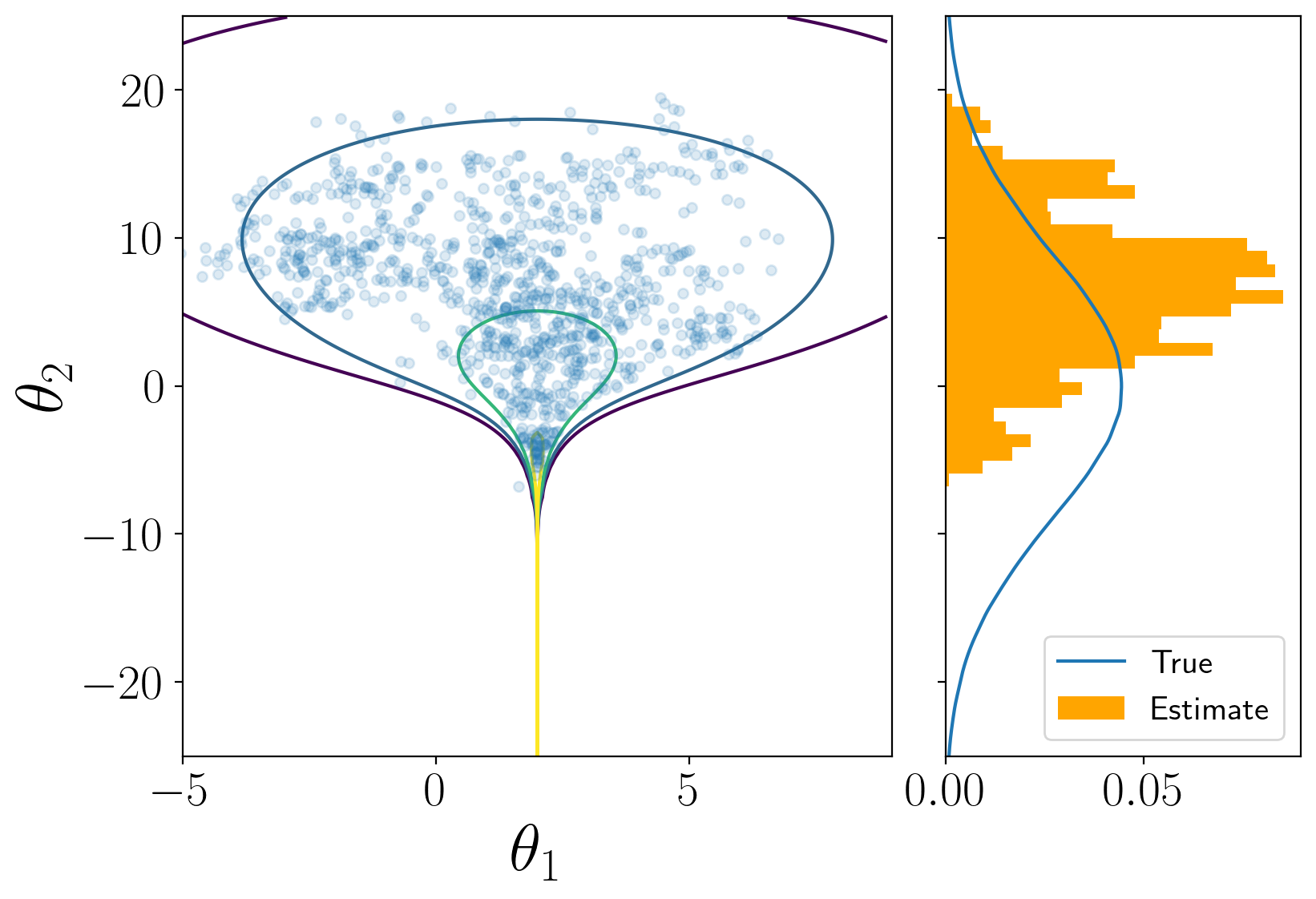}
    \includegraphics[width=0.32\linewidth]{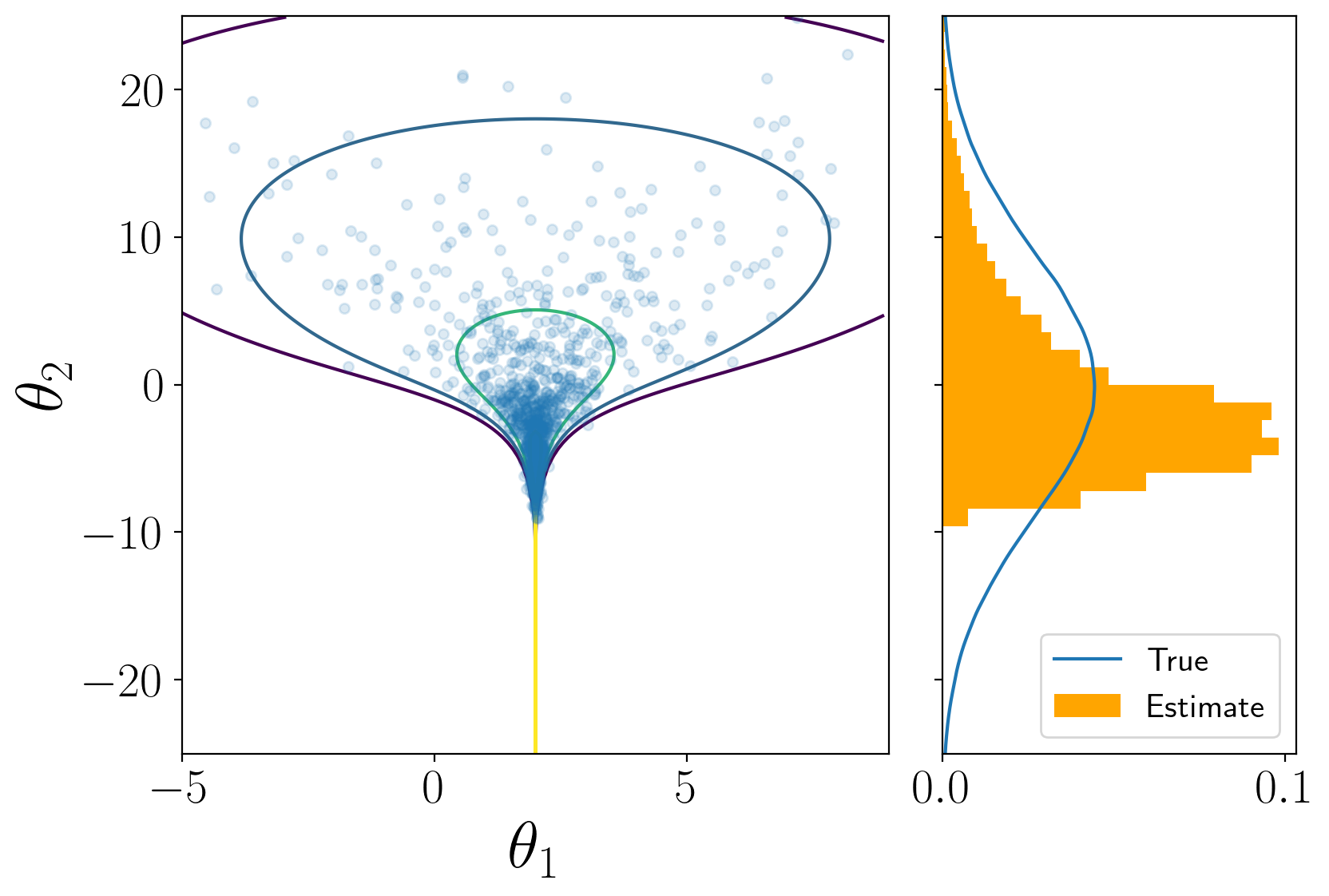}
    \includegraphics[width=0.32\linewidth]{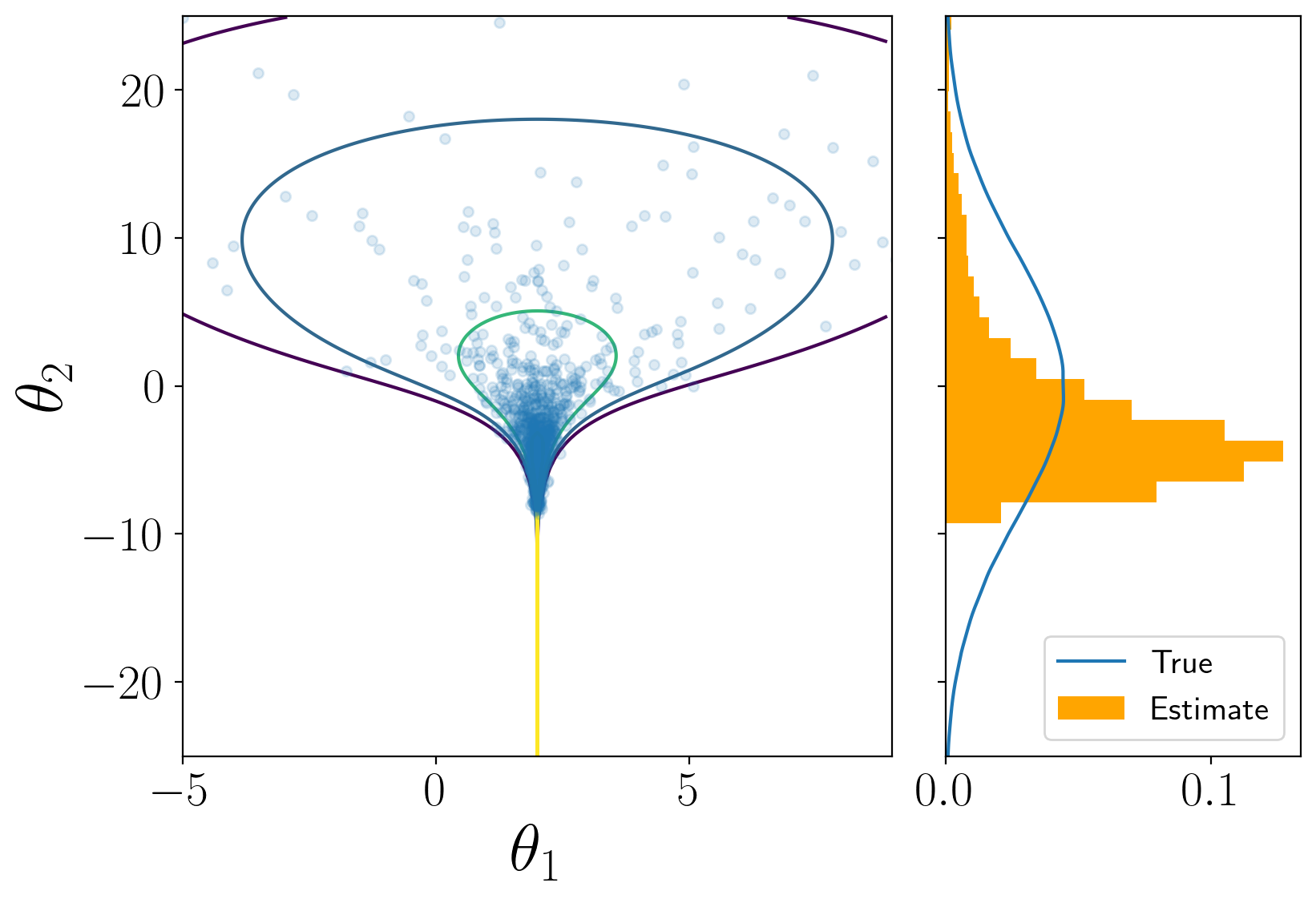}
    \caption{Experiment from \cite{yu2024scalable} demonstrating SLMC on the funnel distribution. \textbf{Left:} SLMC with identity preconditioner \textbf{Middle:} SLMC with adaptive preconditioner based on RMSProp. \textbf{Right:} SLMC with preconditioner based on Adagrad. As we can see, the use of an adaptive preconditioner allows the algorithm to better explore the funnel. We see also that by forcing the preconditioner to be diagonal in RMSProp and align with the axes, the method is able to better explore the funnel.}
    \label{fig:funnel}
\end{figure}

\begin{figure}
    \centering
    \includegraphics[width=0.32\linewidth]{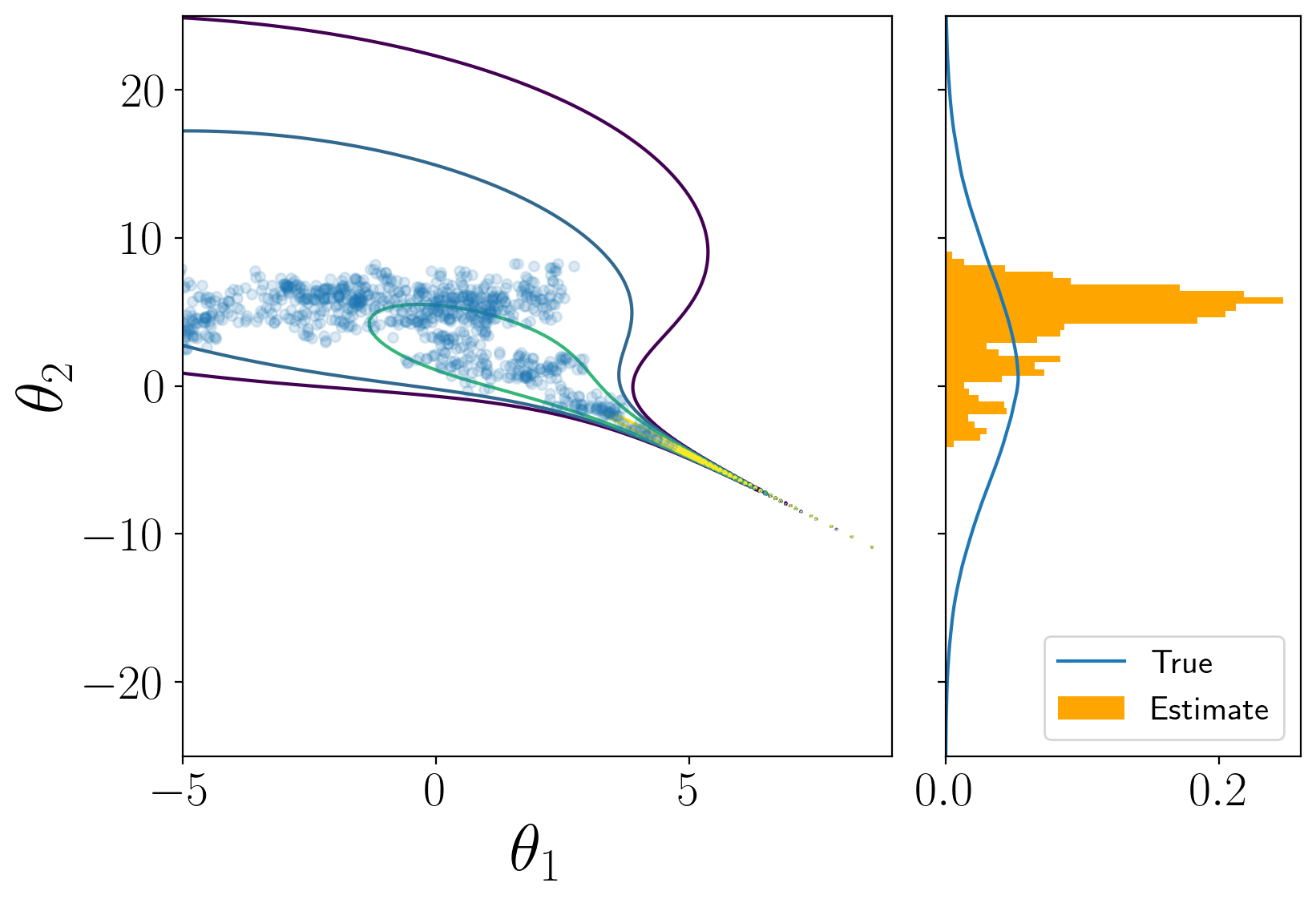}
    \includegraphics[width=0.32\linewidth]{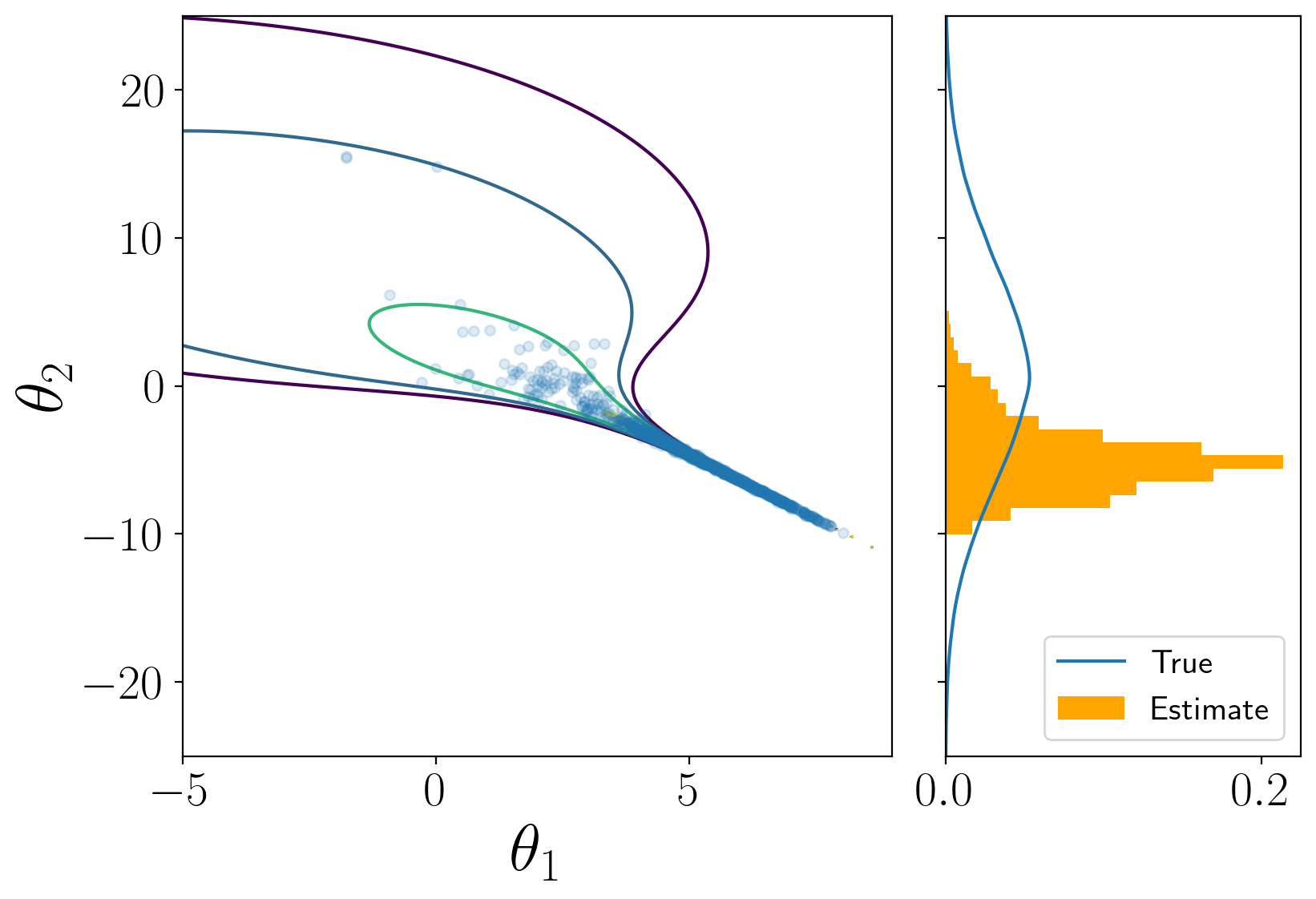}
    \includegraphics[width=0.32\linewidth]{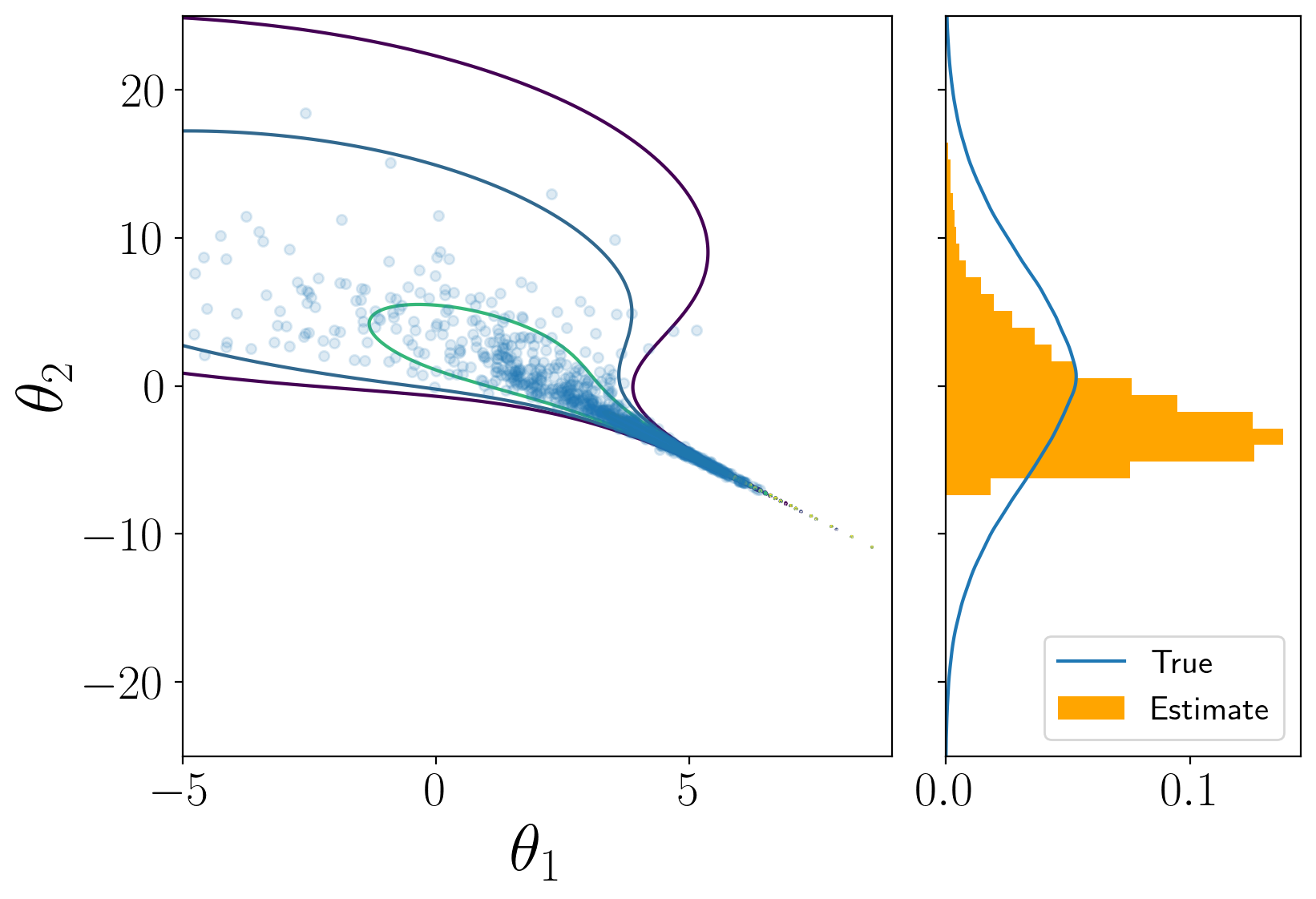}
    \caption{Experiment from \cite{yu2024scalable} demonstrating SLMC on a rotated funnel distribution. \textbf{Left:} SLMC with identity preconditioner \textbf{Middle:} SLMC with adaptive preconditioner based on RMSProp. \textbf{Right:} SLMC with preconditioner based on Adagrad. Here, we see that allowing the adaptation to be non-diagonal allows the method to adapt to the rotated funnel. This allows the SLMC method based on Adagrad to better explore both parts of the funnel.}
    \label{fig:funnelR}
\end{figure}

\section{Discussion}

We have presented Subspace Langevin Monte Carlo, a version of Langevin Monte Carlo that updates along low-dimensional subspaces. This method generalizes Random Coordinate Langevin Monte Carlo, its block coordinate analog, and preconditioned Langevin \cite{bhattacharya2023fast}. Our theoretical results show that through proper adaptation of the selected subspaces and preconditioner, the SLMC method can significantly outperform LMC in terms of overall complexity at a reduced memory cost. Our experiments demonstrate this distinction in practice.

We believe that our work opens many interesting future avenues for research. First, there are many interesting theoretical extensions of SLMC that one can consider. For example, it would be interesting to extend the theory to more general settings, such as convergence under log-Sobolev and Poincar\'e inequalities. It would also be interesting to develop subspace versions of other variances of LMC, such as Underdamped Langevin Monte Carlo. Finally, one can feasibly create preconditioners that adapt to the problem at hand as the method iterates and also study time-varying preconditioners. Another interesting avenue of future research involves the prospect of SLMC as subspace descent in Wasserstein space. It would be interesting to consider subspace descent methods over Wasserstein space for more general functions, where we can view this method as an alternative to sliced Wasserstein gradient flows \cite{bonet2022efficient}. Finally, it would be interesting to incorporate these ideas into the training of score-based diffusion models to further the ideas in put forth in \cite{jing2022subspace} with more principled ways of choosing subspaces. This would open up broader applications for the methods and theory developed in this paper.

\section{Acknowledgements}

This work was funded by the NSF under Award No. 2305315.

\bibliographystyle{plainnat}
\bibliography{refs}

\appendix

\section{Supplementary Proofs}

\subsection{Proof of Propostion \ref{prop:ctssubdescent}}

\begin{proof}
    \begin{align*}
        \partial_t(f(x_t) - f(x_\star)) &=\langle \nabla f(x_t), \bP_t \nabla f(x_t) \rangle \\
        &\leq - m_t \|\nabla f(x_t)\|^2 \\
        &\leq -m_t \alpha (f(x_t) - f(x_\star)).
    \end{align*}
    Therefore by Gr\"onwall's inequality,
    \[
        f(x_t) - f(x_\star) \leq \exp\left( -\int_0^t m_t \alpha dt \right) (f(x_0) - f(x_\star)).
    \]
    If we instead take an expectation, 
    \begin{align*}
        \partial_t(\E f(x_t) - f(x_\star)) &= \E \langle \nabla f(x_t), \bP_t \nabla f(x_t) \rangle \\
        &\leq - c\E \|\nabla f(x_t)\|^2 \\
        &\leq - c\alpha (\E f(x_t) - f(x_\star)).
    \end{align*}
\end{proof}

\subsection{Proof of Lemma \ref{lem:contract}}

\begin{proof}
    
Applying It\^o's lemma to $f(z, z') =  \|z-z'\|_{\bA^{-1}}^2$
\begin{align*}
    f(Z_t, Z_t') - f(Z_0, Z_0') &= 2\int_0^t \langle Z_s - Z_s', \nabla V(Z_s) - \nabla V(Z_s') \rangle ds \\
    &\leq 2\int_0^t m f(Z_s, Z_s') ds.
\end{align*}

Thus we can apply Gr\"onwall's lemma to $ f(Z_s, Z_s')$ to find
\[
     f(Z_t, Z_t') \leq \exp(-2\int_{0}^t  m ds)  f(Z_0, Z_0').
\]
\end{proof}

\subsection{Proof of Theorem \ref{thm:plmc_conv}}
\label{app:plmc_conv}

\begin{proof}
In the case of a fixed $\bA$, this theorem is just a particular case of Theorem 2 in \cite{ahn2021efficient}. We replicate the proof here with our simplified setting, which avoids needing to work with abstract Bregman divergences.

The proof technique works by separating the Langevin update into a gradient update followed by the addition of a Gaussian, which corresponds to a step with respect to the potential energy $\E_\mu V$ followed by a step with respect to the entropy $\E \mu \log \mu$. It then proceeds in three steps: bounding the decrease in the potential energy, showing that the entropy step does not increase the potential energy too much, and then showing that the entropy step decreases the entropy.
    
We first  preconditioned gradient descent, 
\[  
    x_{k+1} = x_k - h \bA \nabla V(x_k).
\]
Assuming that $V$ is $\beta$ smooth relative to $\|\cdot\|_{\bA^{-1}}^2$,
\begin{align*}
    V(x_{k+1}) &\leq V(x_k) + \langle \nabla V(x_k), x_{k+1} - x_k \rangle + \frac{M}{2} \|x_{k+1} - x_k\|^2_{\bA^{-1}} \\
    &= V(x_k) + \Big(\frac{M h^2}{2} -h \Big)\| \nabla V(x_k)\|_{\bA}^2.
\end{align*}
This gives us a relative descent lemma. For example, if one sets $h\leq 1/M$, we get a descent lemma
\[
     \| \nabla V(x_k)\|_{\bA}^2 \leq 2 h M (V(x_{k}) \leq V(x_{k+1})).
\]

Similar to \cite{ambrosio2008gradient,chewi2024log}, when $h \leq 1/M$, we can prove an evolution variational inequality (EVI)
\begin{align*}
    \|x_{k+1} - y \|_{\bA^{-1}}^2&= \|x_k-y\|_{\bA^{-1}}^2 -2h\langle \nabla V(x_k), x_k - y\rangle + h^2 \|\nabla V(x_k)\|_{\bA^{-1}}^2 \\
    &\leq (1-m h) \|x_k-y\|_{\bA^{-1}}^2 -2h(V(x_k)-V(y)) + 2h(V(x_k) - V(x_{k+1}))\\
    &\leq (1-m h) \|x_k-y\|_{\bA^{-1}}^2 - 2h (V(x_{k+1}) - V(y)). 
\end{align*}

Applying this EVI to $X_{kh}^+$, $X_{kh}$ and $Z$, and taking an expectation, we get contraction of the energy with respect to the $W_{2, \bA^{-1}}$ distance
\begin{equation}\label{eq:potbd1}
    \cE(\mu_{kh}^+) - \cE(\pi) \leq \frac{1}{2h} \left[ (1-m h) W_{2, \bA^{-1}}^2(\mu_{kh}, \pi) - W_{2, \bA^{-1}}^2(\mu_{kh}^{+}, \pi)\right].
\end{equation}

We further have that the energy does not increase too much over the entropy step following the same computation in \cite{durmus2019analysis}, which uses the $M$-relative smoothness, 
\begin{equation}\label{eq:potbd2} 
    \cE(\mu_{(k+1)h}) - \cE(\mu_{kh}^+) = M d h.
\end{equation}

Finally, we have that the entropy $\cH$ is relatively convex \cite{ahn2021efficient}. In particular, for $X \sim \mu$ and $Y \sim \nu$ optimally coupled in terms of $W_{2, \bA^{-1}}$, we have
\[
    \cH(\nu) \geq \cH(\mu) + \langle \nabla_{W_2} \cH(\mu)(X), Y-X \rangle.
\]
If we let $X$ be the output of a full step and $Y \sim \pi$, this implies that
\[
    \cH(\pi) \geq \cH(\mu_{(k+1)h}) + \langle \nabla_{W_2} \cH(\mu_{(k+1)h})(X), Y-X \rangle.
\]
Let $Q_t$ denote the semigroup 
\[
    Q_t f(x) = \E f(x + \sqrt{2} \bA^{1/2} B_t),
\]
so that $\mu_{(k+1)h} = \mu_{kh}^+ Q_h$. We claim that $\cH(\mu_{kh}^+ Q_t)$ is the gradient flow of $\cH$ with respect to the $W_{2, \bA^{-1}}$ geometry, and is therefore nonincreasing. Indeed, the $W_{2, \bA^{-1}}$ gradient of $\cH$ is 
\[
    \nabla_{W_{2, \bA^{-1}}} \cH(\mu) = \bA \nabla_{W_2} \log(\mu)
\]
And so the $W_{2,\bA^{-1}}$ gradient flow of $\cH$ is 
\[
    \partial_t \mu_t = \diver(\mu_t \bA \nabla \log(\mu_t).
\]
This is precisely the same as what is found for the Fokker-Planck equation corresponding to the diffusion $\sqrt{2} \bA^{1/2} dB_t$.
Therefore, for $X_{kh+t}^+$ and $Z$ optimally coupled for the $W_{2,\bA^{-1}}$ distance, and since $\cH$ is decreasing along $\mu_{kh}^+ Q_t$,
\[
    \partial_t W_{2,\bA^{-1}}^2(\mu_{kh}^+Q_t, \pi) \leq 2\langle \nabla_{W_2} \cH(\mu_{(k+1)h})(X_{kh+t}^+), Z-X_{kh+t}^+ \rangle \leq \cH(\pi) - \cH(\mu_{(k+1)h}) .
\]
Integrating from $0$ to $h$ yields
\begin{equation}\label{eq:entbd}
    W_{2,\bA^{-1}}^2(\mu_{(k+1)h}, \pi) - W_{2,\bA^{-1}}^2(\mu_{kh}^+, \pi) \leq h [\cH(\pi) - \cH(\mu_{(k+1)h}) ]
\end{equation}

Putting together \eqref{eq:potbd1}, \eqref{eq:potbd2}, and \eqref{eq:entbd},
\begin{equation}\label{eq:bdfixprec}
    W_{2, \bA^{-1}}^2(\mu_{(k+1)h}, \pi) \leq (1-m h) W_{2, \bA^{-1}}^2(\mu_{kh}, \pi) + 2 M d h^2.
\end{equation}

Iterating the inequality,
\begin{align*}
    W_{2, \bA^{-1}}^2(\mu_{Nh}, \pi) &\leq (1-m h)^N W_{2, \bA^{-1}}^2(\mu_{0}, \pi) + \frac{2 M d h}{m}.
\end{align*}

Now in the case of varying $\bA_k$, we note that the proof remains the same, where \eqref{eq:bdfixprec} becomes
\begin{equation}
    W_{2, \bA_k^{-1}}^2(\mu_{(k+1)h}, \pi) \leq (1-m h) W_{2, \bA_k^{-1}}^2(\mu_{kh}, \pi) + 2 M d h^2.
\end{equation}
Using Assumption \ref{assump:precondchange}, we have
\begin{equation}
    W_{2, \bA_{k+1}^{-1}}^2(\mu_{(k+1)h}, \pi) \leq (1-m h) W_{2, \bA_k^{-1}}^2(\mu_{kh}, \pi) + O(h^2) + 2 M d h^2.
\end{equation}
We can again now unroll to find the desired bound.

Note that we need $\bA_k$ to be independent of the particles $X_1, \dots X_k$ in order to take the expectations appropriately in the above proof, hence necessitating Assumption \ref{assump:preconddep}.
\end{proof}

\subsection{Proof of Theorem \ref{thm:conv}}
\label{app:thmproof}

We note that Assumption \ref{assump:relscsm} implies that
\[
    \|\nabla V(x) - \nabla V(y) \|_{\bA_k}^2 \leq M \|y-x\|_{\bA_k^{-1}}^2
\]
for all $k \in \N$.

We begin with a lemma.

\begin{lem} \label{lem:hessexpbd} Suppose that Assumption \ref{assump:relscsm} holds with $\bB = \bA^{-1}$ and Assumption \ref{assump:sm} holds. If $\bW$ is an orthonormal matrix with corresponding relative smoothness parameter $M(\bW)$, then
\[
    \E_{\pi} \| \bW \bW^T \nabla V(Z)  \|_{\bA}^2 \leq M r
\]
\end{lem}
\begin{proof}
    Suppose that we are running \eqref{eq:pld} with the preconditioner $\bW \bW^T \bA \bW \bW^T$. Let $\cL$ be the generator of this process, which takes the form
    \[
        \cL f = \Tr(\bW \bW^T \bA \bW \bW^T \nabla^2 V) - \langle \bW \bW^T \bA \bW \bW^T \nabla V, \nabla f \rangle. 
    \]
    Since $\pi \propto \exp(-V)$ is stationary, we have
    \[
        0 = \E_{\pi} \cL V = \E \Tr(\bW \bW^T \bA \bW \bW^T \nabla^2 V) - \langle \bW \bW^T \bA \bW \bW^T \nabla V, \nabla V \rangle.
    \]
    We have 
    \begin{align*}
        \Tr(\bW \bW^T \bA \bW \bW^T \nabla^2 V) &= \Tr((\bW \bW^T \bA \bW \bW^T)^{1/2} \nabla^2 V (\bW \bW^T \bA \bW \bW^T)^{1/2}) \\
        &\leq M(\bW) \Tr((\bW \bW^T \bA \bW \bW^T)^{1/2} \bA^{-1} (\bW \bW^T \bA \bW \bW^T)^{1/2}) \\
        &= M(\bW) \Tr([\bW^T \bA^{-1} \bW] [\bW^T \bA \bW] ) \\
        &= M(\bW) r
    \end{align*}
\end{proof}

We begin now restate our main theorem.
\begin{thm}
    Under Assumptions \ref{assump:relscsm}-\ref{assump:sm}, and $h \leq \min\left(\min \phi_{ik}/M, \right)$
    \begin{align*}
        W_{2,\bA_N^{-1}}(\mu_N, \pi) &\lesssim \exp(-\frac{h m N}{4}) W_{2,\bA_0^{-1}}(\mu_0, \pi) +\frac{\sqrt{r  h }}{ m } \sqrt{\sum_i   \frac{M_{ik}^2}{\phi_{ik}}} .
    \end{align*}
\end{thm}
\begin{proof}
Let $\bP_k = \bW_{ik}  \bD_{ik} \bW_{ik}^T$ be the projection chosen at the $k$th iteration, where $\bW_{ik} \in O(d,r)$ and $\bD_{ik}$ is diagonal.

We define the auxiliary process
\begin{equation}
    Z_{k+1} = Z_k(h) = Z_k(0) - \int_{0}^{h_{ik}} \bP_k \nabla V(Z_k(s)) ds +  \sqrt{2h} \bP_k^{1/2} \xi_k.
\end{equation}
Here, $\xi_k$ is the same noise as \eqref{eq:slmc}.
Let $\Delta_k = Z_k - X_k$.
We have
\[
    \Delta_{k+1} = \Delta_{k} - \bP_k \Big[\int_{0}^{h_{ik}}  \nabla V(Z_k(s)) ds -  \nabla V(X_k) \Big].
\]
After multiplying by $\bW_k^T$, we see that really we are looking at a difference in block-coordinate updates between $Z_k$ and $X_k$ in the new basis:
\[
    \bW_{ik}^T \Delta_{k+1} = \bW_{ik}^T\Delta_{k} -  \bD_{ik} \bW_{ik}^T \Big[\int_{0}^{h_{ik}}  \nabla V(Z_k(s)) ds -  \nabla V(X_k) \Big].
\]

By assumption, $\bW_k$ is the full basis from which we select $\bW_{ik}$, which is divided into $d/r$ blocks, each of which we select with probability $\phi_{ik}$. These blocks are denoted by $\bW_{ik}$. 
We have
\[
    \E \|\bW_{ik} \bW_{ik}^T \Delta_{k+1}\|_{\bA_k^{-1}}^2 = \phi_{ik} \E [ \|\bW_{ik}\bW_{ik}^T \Delta_{k+1}\|_{\bA_k^{-1}}^2 | i_k = i] + (1-\phi_{ik}) \E \|\bW_{ik}\bW_{ik}^T \Delta_k\|_{\bA_k^{-1}}^2.
\]
Going term by term, given $i$ is chosen,
\begin{align} \label{eq:young1}
    \E [ \| \bW_{ik}\bW_{ik}^T \Delta_{k+1}\|_{\bA_k^{-1}}^2 ] = &(1+a) \E \| \bW_{ik}\bW_{ik}^T\Delta_{k+1} + \\ \nonumber  
    &\bW_{ik}\bD_{ik} \bW_{ik}^T \int_0^{h_{ik}} \nabla V(Z_k(s) - \nabla V(Z_k) ds \|_{\bA_k^{-1}}^2 + \\ \nonumber
    &\Big(1 + \frac{1}{a} \Big) \E \| \bW_{ik}\bD_{ik} \bW_{ik}^T \int_0^{h_{ik}} \nabla V(Z_k(s) - \nabla V(Z_k) ds\|_{\bA_k^{-1}}^2.
\end{align}
The first term of this is then bounded by
\begin{align*}
    \E \|  \bW_{ik} \bW_{ik}^T  \Delta_{k+1} + & \bW_{ik} \bD_{ik} \bW_{ik}^T  \int_0^{h_{ik}} \nabla V(Z_k(s) - \nabla V(Z_k) ds \|^2 \leq \E \|\bW_{ik} \bW_{ik}^T  \Delta_{k}\|_{\bA_k^{-1}}^2 \\
    &-2 h_{ik} \E \langle \bW_{ik}^T \Delta_k,  \bD_{ik} \bW_{ik}^T (\nabla V(Z_k) - \nabla V(X_k)) \rangle_{\bA_k^{-1}} \\
    &+ h_{ik}^2 \E \|\bW_{ik} \bD_{ik} \bW_{ik}^T (\nabla V(Z_k) - \nabla V(X_k)) \|_{\bA_k^{-1}}^2.
\end{align*}
Again, given $i$ is chosen, following the argument in the proof of Lemma 12 of \cite{ding2021random}, and using Lemma \ref{lem:hessexpbd}, while being careful with our notion of relative smoothness, the second term of \eqref{eq:young1} is bounded by
\begin{align*}
    \E &\|\bW_{ik} \bD_{ik} \bW_{ik}^T  \int_0^{h_{ik}} \nabla V(Z_k(s)) - \nabla V(Z_k) \di s\|_{\bA^{-1}}^2 \\
    &\leq h_{ik} \int_0^{h_{ik}} \E \| \bW_{ik}\bD_{ik} \bW_{ik}^T (\nabla V(Z_k(s)) - \nabla V(Z_k)) \|_{\bA^{-1}}^2 \di s
    \\ 
    &= h_{ik} \int_0^{h_{ik}} \E \| \bW_{ik}\bW_{ik}^T (\nabla V(Z_k(s)) - \nabla V(Z_k)) \|_{\bA}^2 \di s
    \\  
    &\leq h_{ik} M_{ik}^2 \int_0^{h_{ik}}  \E \| Z_k(s) -  Z_k \|_{\bA^{-1}}^2 \di s\\
    &= h_{ik} M_{ik}^2 \int_0^{h_{ik}}  \E \| \int_{0}^s \bP_k \nabla V(Z_k(\tau)) \di \tau +  \sqrt{2} \bP_k^{1/2} \xi_k \|_{\bA^{-1}}^2 \di s\\
    &\leq 2h_{ik} M_{ik}^2 \int_0^{h_{ik}}  \E \| \int_{0}^s \bP_k \nabla V(Z_k(\tau)) \di \tau \|_{\bA^{-1}}^2 \di s+2h_{ik} M_{ik}^2 \int_0^{h_{ik}}  \E \|    \sqrt{2h_{ik}} \bP_k^{1/2} \xi_k \|_{\bA^{-1}}^2 \di s\\
    &\leq 2h_{ik}^2 M_{ik}^2 \int_0^{h_{ik}} \int_{0}^s \E \|  \bP_k \nabla V(Z_k(\tau))  \|_{\bA^{-1}}^2 \di \tau \di s+ 4r h_{ik}^3 M_{ik}^2   \\
    &\leq  h_{ik}^4 M_{ik}^2 \E_{\pi} \| \bW_{ik} \bW_{ik}^T \nabla V(Z)  \|_{\bA}^2 +4r h_{ik}^3 M_{ik}^2   \\
    &\leq   h_{ik}^4 M_{ik}^3 r + 4r  h_{ik}^3 M_{ik}^2.
\end{align*}
Here, we use the fact that $\bP_k^{1/2} \bA_k^{-1} \bP_k^{1/2} = \bW_{ik} \bW_{ik}^T$ and $\E_{\pi} \| \bW_{ik} \bW_{ik}^T \nabla V(Z)  \|_{\bA_k}^2 \leq M_{ik} r$ by Lemma \ref{lem:hessexpbd}.
Using $h_{ik} = h/\phi_{ik}$, we can put these together to yield
\begin{align*}
    \E \|\bW_{ik} \bW_{ik}^T \Delta_{k+1} \|_{\bA^{-1}}^2 &\leq (1+a \phi_{ik}) \E \|\bW_{ik} \bW_{ik}^T \Delta_{k}\|_{\bA_k^{-1}}^2 \\
    &-2 (1+a) h \E \langle \bW_{ik} \bW_{ik}^T \Delta_k,  \bW_{ik} \bD_{ik} \bW_{ik}^T (\nabla V(Z_k) - \nabla V(X_k)) \rangle_{\bA_k^{-1}}  \\
    &+ (1+a) \frac{h^2}{\phi_{ik}} \E \|  \bW_{ik} \bD_{ik} \bW_{ik}^T (\nabla V(Z_k) - \nabla V(X_k))\|_{\bA_k^{-1}}^2 \\
    &+\Big( 1+ \frac{1}{a} \Big) \phi_{ik} [ \frac{r h^4 M_{ik}^3 r}{\phi_{ik}^4} + \frac{4r  h^3 M_{ik}^2}{\phi_{ik}^3}]
\end{align*}

Letting $a = h m /  \phi_{ik}$, $1+1/a \lesssim \phi_{ik}/h m$, $h \leq \min \phi_{ik}/M$,
\begin{align*}
    \E \|\bW_{ik} \bW_{ik}^T \Delta_{k+1} \|_{\bA^{-1}}^2 &\leq (1+h m) \E \|\bW_{ik} \bW_{ik}^T \Delta_{k}\|_{\bA_k^{-1}}^2 \\
    &-2 (1+h m/\phi_{ik}) h \E \langle \bW_{ik} \bW_{ik}^T \Delta_k,  \bW_{ik} \bD_{ik} \bW_{ik}^T (\nabla V(Z_k) - \nabla V(X_k)) \rangle_{\bA_k^{-1}}  \\
    &+ (1+h m/\phi_{ik}) \frac{h^2}{\phi_{ik}} \E \|  \bW_{ik} \bD_{ik} \bW_{ik}^T (\nabla V(Z_k) - \nabla V(X_k))\|_{\bA_k^{-1}}^2 \\
    &+\frac{1}{ m} [ \frac{r h^3 M_{ik}^3}{\phi_{ik}^2} + \frac{4r  h^2 M_{ik}^2}{\phi_{ik}}] \\
    &\leq (1+h m) \E \|\bW_{ik} \bW_{ik}^T \Delta_{k}\|_{\bA_k^{-1}}^2 \\
    &-2  \E \langle \bW_{ik}\bW_{ik}^T \Delta_k,  \bW_{ik}\bD_{ik} \bW_{ik}^T (\nabla V(Z_k) - \nabla V(X_k)) \rangle_{\bA_k^{-1}}  \\
    &+ \frac{2}{\phi_{ik}} \E \| \bW_{ik}\bD_{ik} \bW_{ik}^T (\nabla V(Z_k) - \nabla V(X_k))\|_{\bA_k^{-1}}^2 \\
    &+\frac{1}{ m} [ \frac{r h^3 M_{ik}^3}{\phi_{ik}^2} + \frac{4r  h^2 M_{ik}^2}{\phi_{ik}}] 
\end{align*}

Summing, using $\bA_k = \sum_i  \bW_{ik} \bD_{ik} \bW_{ik}^T$,
\begin{align*}
    \E \|\Delta_{k+1} \|_{\bA_k^{-1}}^2 &\leq (1+h m) \E \|\Delta_{k}\|_{\bA_k^{-1}}^2 \\
    &-2 h \E \langle  \Delta_k,  \nabla V(Z_k) - \nabla V(X_k) \rangle \\
    &+ \frac{2 h^2}{\min \phi_{ik}} \E \| (\nabla V(Z_k) - \nabla V(X_k))\|^2_{\bA_k} \\
    &+\frac{1}{ m}\sum_i [ \frac{r h^3 M_{ik}^3}{\phi_{ik}^2} + \frac{4r  h^2 M_{ik}^2}{\phi_{ik}}] 
\end{align*}
The smoothness in Assumption \ref{assump:relscsm} implies that
\begin{align*}
     \| \nabla V(Z_{k})-\nabla V(X_{k})\|_{\bA_k}^2 &\leq M \| Z_{k}-X_{k}\|_{\bA_k^{-1}}^2.
\end{align*}
Therefore, using the previous display and using strong relative monotonicity that is implied by Assumption \ref{assump:relscsm},
\[
    \langle \nabla V(x) - \nabla V(y), x-y \rangle \geq m \|x-y\|_{\bA_k^{-1}}^2,
\]
we find
\begin{align*}
    \E \|\Delta_{k+1} \|_{\bA_k^{-1}}^2 &\leq (1+h m-2h m + \frac{2}{\min \phi_{ik}}  h^2 M) \E \|\Delta_{k}\|_{\bA_k^{-1}}^2 \\
    &+\frac{1}{ m}\sum_i [ \frac{r h^3 M_{ik}^3}{\phi_{ik}^2} + \frac{4r  h^2 M_{ik}^2}{\phi_{ik}}] 
\end{align*}
Since $h \lesssim \frac{\min \phi_{ik}}{M}$,
\begin{align*}
    \E \|\Delta_{k+1} \|_{\bA_k^{-1}}^2 &\leq (1-\frac{h m}{2}) \E \|\Delta_{k}\|_{\bA_k^{-1}}^2 +\frac{1}{ m}\sum_i [  \frac{4r  h^2 M_{ik}^2}{\phi_{ik}}].
\end{align*}
Using Assumption \ref{assump:precondchange} and unrolling,
\begin{align*}
    \E \|\Delta_{k+1} \|_{\bA_{k+1}^{-1}}^2 &\lesssim (1-\frac{h m}{2})^k \E \|\Delta_{0}\|_{\bA_0^{-1}}^2 +\frac{1}{ m}\sum_j (1-\frac{h m}{2})^{k-j} \sum_i [  \frac{4r  h^2 M_{ij}^2}{\phi_{ij}}]  .
\end{align*}
Optimally coupling $\Delta_0$ yields the result.

Again, throughout, we used the fact that Assumption \ref{assump:preconddep} implies that $\bA_k$ is independent of the particles $X_1, \dots X_k$ to take the expectations appropriately.
% and letting $j$ be the index where $ \sum_i [  \frac{4r  h^2 M_{ij}^2}{\phi_{ij}}] $ is maximized,
% \begin{align*}
%    W_{2,\bA_N^{-1}}(\mu_N, \pi) &\lesssim \exp(-\frac{h m N}{4}) W_{2,\bA_0^{-1}}(\mu_0, \pi) +\frac{\sqrt{r  h }}{ m } \sqrt{\sum_i   \frac{M_{ij}^2}{\phi_{ij}}} .
% \end{align*}
\end{proof}

\end{document}